\documentclass[letterpaper, 10 pt, conference]{IEEEtran}  

\IEEEoverridecommandlockouts                              

\usepackage{verbatim}
\usepackage{cite}
\usepackage{multicol}
\usepackage{amsmath}
\usepackage{amssymb}
\usepackage{graphicx}
\usepackage{epstopdf}
\usepackage{caption}
\usepackage{subcaption}
\usepackage{tikz}
\usepackage{algorithm}
\usepackage{algorithmic}
\newtheorem{theorem}{Theorem}
\newtheorem{lemma}{Lemma}
\def\network-oblivious{individually consistent}
\title{Collaborative Learning of Stochastic Bandits over a Social Network}

\author{Ravi Kumar Kolla, Krishna Jagannathan and Aditya Gopalan
\thanks{R.~Kolla, K.~Jagannathan are with the Department of Electrical Engineering, IIT
Madras, Chennai, India 600036. Email: {\tt \{ee12d024,   krishnaj\}@ee.iitm.ac.in}. A.~Gopalan is with the Department of Electrical Communication Engineering, IISc, Bangalore, India 560012. Email: {\tt aditya@ece.iisc.ernet.in}
        }}

\begin{document}

\maketitle
\thispagestyle{empty}
\pagestyle{empty}

\begin{abstract}
  We consider a collaborative online learning paradigm, wherein a
  group of agents connected through a social network are engaged in
  playing a stochastic multi-armed bandit game. Each time an agent
  takes an action, the corresponding reward is instantaneously
  observed by the agent, as well as its neighbours in the social
  network. We perform a regret analysis of various policies in this
  collaborative learning setting. A key finding of this paper is that
  natural extensions of widely-studied single agent learning policies
  to the network setting need not perform well in terms of regret. In
  particular, we identify a class of non-altruistic and
  individually consistent policies, and argue by deriving regret lower
  bounds that they are liable to suffer a large regret in the
  networked setting. We also show that the learning performance can be
  substantially improved if the agents exploit the structure of the
  network, and develop a simple learning algorithm based on
  dominating sets of the network. Specifically, we first consider a
  star network, which is a common motif in hierarchical social
  networks, and show analytically that the hub agent can be used as an
  information sink to expedite learning and improve the overall
  regret. We also derive network-wide regret bounds for the
  algorithm applied to general networks. We conduct numerical
  experiments on a variety of networks to corroborate our analytical
  results.
\end{abstract}
\IEEEpeerreviewmaketitle
\section{Introduction}
\label{Introduction}
We introduce and study a collaborative online learning paradigm,
wherein a group of agents connected through a social network are
engaged in learning a stochastic Multi-Armed Bandit (MAB) problem. In this setting, a set of agents are connected by a graph, representing an information-sharing network among them. At each time, each agent (a node in the social network graph) chooses an action (or \emph{arm}) from a finite set of actions, and receives a stochastic reward corresponding to the chosen arm, from an
unknown probability distribution. In addition, each agent shares the action index and the corresponding reward sample instantaneously with its neighbours in the graph. The agents are interested in maximising (minimising) their net cumulative reward (regret) over time. When there is only one learning agent, our setting is identical to the classical multi-armed bandit problem, which is a widely-studied framework for sequential learning \cite{Lai,Auer}.

Our framework is motivated by scenarios that involve multiple decision
makers acting under uncertainty towards optimising a common goal. One
such example is that of a large-scale distributed recommendation
system, in which a network of backend servers handles user traffic in a
concurrent fashion. Each user session is routed to one of the servers
running a local recommendation algorithm. Due to the high volume of
recommendation requests to be served, bandwidth and computational constraints may
preclude a central processor from having access to the 
observations from all sessions, and issuing recommendations
simultaneously to them in real time. In this situation, the servers
must resort to using low-rate information from their neighbours to
improve their learning, which makes this a collaborative networked
bandit setting.

Another application scenario is that of cooperative transportation routing
with mobile applications that provide social network overlays,  like
Waze~\cite{waze}. A user in this system is typically interested in
taking the fastest or most efficient route through a city, with her
app offering a choice of routes, and also recording observations
from past choices. In addition, users can also add other trusted users as
friends, whose observations then become available as additional
information for future decision making. The social network among the
users thus facilitates local information exchange, which could
help users optimise their future decisions (choices of routes) faster.

In our setting, the agents use their social network to aid their
learning task, by sharing their action and reward samples with their
immediate neighbours in the graph. It seems reasonable that this
additional statistical information can potentially help the agents to
optimize their rewards faster than they would if they were completely
isolated. Indeed, several interesting questions arise in this
collaborative learning framework. For example, how does the structure
of the social network affect the rate at which the agents can learn?
Can good learning policies for the single agent setting be extended
naturally to perform well in the collaborative setting? Can agents
exploit their `place' in the network to learn more efficiently? Can
`more `privileged' agents (e.g., nodes with high degree or influence)
help other agents learn faster? This work investigates and answers
some of these questions analytically and experimentally.

\subsection{Our Contributions} 
We consider the collaborative bandit learning scenario, and analyse the total regret incurred by the agents (regret of the network) over a long but finite horizon $n$. Our specific contributions in this paper are as follows. 

We first introduce and analyse the expected regret of the UCB-Network
policy, wherein all the agents employ an extension of the celebrated
UCB1~\cite{Auer} policy. In this case, we derive an upper bound on the
expected regret of a generic network. The upper bound involves a
graph-dependent constant, which is obtained as the solution to a
combinatorial optimisation problem. We then specialize the upper bound
to common network topologies such as the fully connected and the star graphs, in
order to highlight the impact of the social network structure on the
derived upper bound.

Second, we derive a universal lower bound on the expected regret of a generic network, for a large class of `reasonable' policies. This lower bound is based on fundamental statistical limits on the learning rate, and is independent of the network structure. To incorporate the network structure, we derive another lower bound on the expected regret of a generic network, as a function of a graph dependent parameter. This bound holds for the class of  \emph{non-altruistic and \network-oblivious} (NAIC) policies, which includes appropriate extensions of well-studied single agent learning policies, such as UCB1~\cite{Auer} and Thompson sampling~\cite{Shipra} to a network setting. We then observe that the gap between the derived lower bound for the NAIC class of policies, and the upper bound of the UCB-Network policy can be quite large, even for a simple star network\footnote{Our special interest in star graphs is motivated by the fact that social networks often posses a hub-and-spoke structure, where the star is a commonly occurring motif.}.     

Third, we consider the class of star networks, and derive a refined
lower bound on the expected regret of a large star network for NAIC
policies. We observe that this refined lower bound matches (in an
order sense) the upper bound of the UCB-Network. We thus
conclude that widely-studied sequential learning policies (NAIC) which
perform well in the single agent setting, may perform poorly in terms
of the expected regret of the network when used in a network setting,
especially when the network is highly hierarchical.

Next, motivated by the intuition built from our bounds, we seek
policies which can exploit the social network structure in order to
improve the learning rates. In particular, for an $m$-node star
network, we propose a Follow Your Leader (FYL) policy, which exploits
the centre node's role as an `information hub'. We show that the
proposed policy suffers a regret which is smaller by a factor of $m$
compared to that of any NAIC policy. In particular, the network-wide
regret for the star-network under the FYL policy matches (in an order
sense) the universal lower bound on regret. This serves to confirm
that using the centre node's privileged role is the right information
structure to exploit in a star network.

Finally, we extend the above insights to a generic network. To this end, we make a connection between the smallest \emph{dominating set} of the network, and the achievable regret under the FYL policy. In particular, we show that the expected regret of the network is upper bounded by the product of the \emph{domination number} and the expected regret of a single isolated agent.

In sum, our results on the collaborative bandit learning show that policies that exploit the network structure often suffer substantially lesser expected regret, compared to single-agent policies extended to a network setting. 

\subsection{Related Work}
\label{Relatedwork}

There is a substantial body of work that deals with the learning of various types of single agent MAB problems~\cite{Agarwal,Auer,Lai,Auer3,Bubeck}. However, there is relatively little work on the learning of stochastic MAB problems by multiple agents. Distributed learning of a MAB problem by multiple agents has been studied in the context of a cognitive radio frame work in~\cite{Keqin,Anima,Naumann}. Unlike these models, a key novelty in our model is that it incorporates information sharing among the agents since they are connected by a network.
In ~\cite{Shuang}, the authors assume that each player, in each round, has access to the entire history corresponding to the actions and the rewards of all users in the network -- this is a special case of our generic user network model. In \cite{cesa2016delay}, the authors deal with the learning of adversarial MAB problem by multiple agents connected through a network.

The primary focus in~\cite{Buccapatnam2} is centralized learning, wherein an external agent chooses the actions for the users in the network. 
The learning of the stochastic MAB problem by multiple users has also been addressed from a game-theoretic perspective in~\cite{Buccapatnam1}; the randomised algorithm proposed therein uses the parameters of the MAB problem, which are unknown to the algorithm in practice. In contrast, we propose deterministic algorithms that do not require these parameters.

 In a class of MAB problems considered in ~\cite{caron2012leveraging,Mannor,Alon}, a sole learning agent receives side observations in each round from \emph{other arms}, in addition to samples from the chosen arm. Another related paper is~\cite{kar} -- here, the model consists of a single major bandit (agent)  and a set of minor bandits. While the major bandit observes its rewards, the minor bandits can only observe the actions of the major bandit. However, the bandits are allowed to exchange messages with their neighbours, to receive the reward information of the major bandit. Clearly, the models described above are rather different from the setting we consider in this work.

{\em Organization.} We describe the system model in
Section~\ref{Model}. Section~\ref{UCB-Network} presents the regret
analysis of the UCB-Network policy. Lower bounds
on the expected regret of the network under certain classes of
policies are presented in Section~\ref{Lowerbound1}. Section~\ref{FYL} presents the regret analysis of the FYL policy. Numerical results are presented in
Section~\ref{Results}, and Section~\ref{Conclusions} concludes the
paper. 
\section{System Model}
\label{Model}
We first briefly outline the single agent stochastic MAB problem. Let $\mathcal{K} = \lbrace 1, 2,\dots, K \rbrace$ be the set of arms available to the agent. Each arm is associated with a distribution, independent of others, say $\mathcal{P}_1, \mathcal{P}_2, \dots, \mathcal{P}_K$, and let $\mu_1,\mu_2, \dots,\mu_K$ be the corresponding means, unknown to the agent. Let $n$ be the time horizon or the total number of rounds. In each round $t$, the agent chooses an arm, for which he receives a reward, an i.i.d. sample drawn from the chosen arm's distribution. The agent can use the knowledge of the chosen arms and the corresponding rewards upto round $(t-1)$ to select an arm in round $t.$ The goal of the agent is to maximize the cumulative expected reward up to round $n$. 

Now, we present the model considered in this paper. We consider a set of users $V$ connected by an undirected fixed network $G=(V,E)$ \footnote{We use the adjacency matrix $A$ to represent the network $G$. If $(i,j) \in E$ then $A(i,j) = A(j,i) = 1$, otherwise $A(i,j) = A(j,i) =0$. We assume that $A(i,i) = 1\,\, \forall i \in V$.}, with $| V | = m$. Assume that each user is learning the same stochastic MAB problem i.e., faces a choice in each time from among the same set of arms $\mathcal{K}$. In the $t^{th}$ round, each user $v$ chooses an arm, denoted by $a^v(t) \in \mathcal{K}$, and receives a reward, denoted by $X^v_{a^v(t)}(t)$, an i.i.d. sample drawn from $\mathcal{P}_{a^v(t)}$. In the  stochastic MAB problem set-up, for a given user $v$, the rewards from arm $i$, denoted by $\{X^v_i(t): t = 1, 2, \ldots \}$, are i.i.d. across rounds. Moreover, the rewards from distinct arms $i$ and $j$, $X^v_i(t)$, $X^v_j(s)$, are independent. If multiple users choose the same action in a certain round, then each of them gets an independent reward sample drawn from the chosen arm's distribution. We use the subscripts $i$, $v$ and $t$ for arms, nodes and time respectively. The information structure available to each user is as follows. A user $v$ can observe the actions and the respective rewards of itself and its one hop neighbours in round $t$, before deciding the action for round $(t+1)$.   

The policy $\Phi^v$ followed by a user prescribes actions at each time $t,$ $\Phi^v(t): H^v(t) \rightarrow \mathcal{K},$ where $H^v(t)$ is the information available with the user till round $t.$ A policy of the network $G,$ denoted by $\Phi,$ comprises of the policies pertaining to all users in $G.$ The performance of a policy is quantified by a real-valued random variable, called \textit{regret}, defined as follows. The regret incurred by user $v$ for using the policy $\Phi^v$ upto round $n$ is defined as,
\begin{equation*}
\label{eq:2.1}
R^v_{\Phi}(n) = \sum\limits_{t=1}^n \left( \mu^* - \mu_{a^v(t)} \right)  = n \mu^* - \sum\limits_{t=1}^n \mu_{a^v(t)},
\end{equation*}
where $a^v(t)$ is the action chosen by the policy $\Phi^v$ at time $t$, and $\mu^* = \max\limits_{1 \leq i \leq K} \mu_i$. We refer to the arm with the highest expected reward as the optimal arm. The regret of the entire network $G$ under the policy $\Phi$ is denoted by $R^G_{\Phi}(n)$, and is defined as the sum of the regrets of all users in $G$. The expected regret of the network is given by: 
\begin{equation}
\label{eq:2.3}
\mathbb{E} [ R^G_{\Phi}(n) ] = \sum \limits_{v \in V} \sum\limits_{i=1}^K \Delta_i \mathbb{E}  [ T_i^v(n) ],
\end{equation}
where $\Delta_i = \mu^* - \mu_i$, and $T_i^v(n)$ is the number of times arm $i$ has been chosen by $\Phi^v$ upto round $n$. We omit $\Phi$ from the regret notation, whenever the policy can be understood from the context. Our goal is to devise learning policies in order to minimise the expected regret of the network. 

Let $\mathcal{N}(v)$ denote the set consisting of the node $v$ and its one-hop neighbours. Let $m^v_i(t)$ be the number of times arm $i$ has been chosen by node $v$ and its one-hop neighbours till round $t$, and  $\hat{\mu}_{m_i^v(t)}$ be the average of the corresponding reward samples. These are given as:
\begin{align*}
m^v_i(t)  &= \sum\limits_{u \in \mathcal{N}(v)} T_i^u(t) \\
\hat{\mu}_{m_i^v(t)} &=\frac{1}{{m_i^v(t)}} \sum\limits_{u \in \mathcal{N}(v)} \sum\limits_{k=1}^{t} X_{a^u(k)}^u(k) \mathbb{I} \lbrace a^u(k) = i \rbrace,
\end{align*}  
where $\mathbb{I}$ denotes the indicator function. We use $m^G_i(t)$ to denote the number of times arm $i$ has been chosen by all nodes in the network till round $t$.

\begin{figure}
\begin{center}
\vspace{0.25cm}	
\begin{tikzpicture}[thick, scale=0.35]	
	
\draw[] (1,0) circle (.3); \node at (1,0) {\tiny 1};
\draw[] (3,0) circle (.3); \node at (3,0) {\tiny 2};
\draw[] (0,2) circle (.3); \node at (0,2) {\tiny 5};
\draw[] (4,2) circle (.3); \node at (4,2) {\tiny 3};
\draw[] (2,4) circle (.3); \node at (2,4) {\tiny 4};
\draw[line width=.01cm] [-](1.15,0)--(2.85,0);
\draw[line width=.01cm] [-](0.9,0.1)--(0,1.85);
\draw[line width=.01cm] [-](3.1,0.1)--(4,1.85);
\draw[line width=.01cm] [-](0,2.15)--(1.85,4);
\draw[line width=.01cm] [-](2.2,4)--(4,2.15);
\draw[line width=.01cm] [-](2.85,0.1)--(0.15,2);
\draw[line width=.01cm] [-](1.15,0.1)--(3.85,2);
\draw[line width=.01cm] [-](0.15,2)--(3.85,2);
\draw[line width=.01cm] [-](2,3.85)--(1,0.15);
\draw[line width=.01cm] [-](2,3.85)--(3,0.15);
\draw[line width=.01cm] [-](-0.5,-0.5)--(-0.5,4.5);
\draw[line width=.01cm] [-](4.5,-0.5)--(4.5,4.5);
\draw[line width=.01cm] [-](-0.5,-0.5)--(4.5,-0.5);
\draw[line width =.01cm][-](-0.5,4.5)--(4.5,4.5);
\node at (2, -1) {(a)};

\draw[] (7,0) circle (.3); \node at (7,0) {\tiny 1};
\draw[] (9,0) circle (.3); \node at (9,0) {\tiny 2};
\draw[] (6,2) circle (.3); \node at (6,2) {\tiny 5};
\draw[] (10,2) circle (.3); \node at (10,2) {\tiny 3};
\draw[] (8,4) circle (.3); \node at (8,4) {\tiny 4};
\draw[line width=.01cm] [-](7.15,0)--(8.85,0);
\draw[line width=.01cm] [-](6.85,0.15)--(6,1.85);
\draw[line width=.01cm] [-](9.15,0.15)--(10,1.85);
\draw[line width=.01cm] [-](6,2.15)--(7.85,4);
\draw[line width=.01cm] [-](8.15,4)--(9.85,2.15);
\draw[line width=.01cm] [-](5.5,-0.5)--(5.5,4.5);
\draw[line width=.01cm] [-](5.5,4.5)--(10.5,4.5);
\draw[line width=.01cm] [-](10.5,4.5)--(10.5,-0.5);
\draw[line width =.01cm][-](10.5,-0.5)--(5.5,-0.5);
\node at (8, -1) {(b)};

\draw[] (12,0) circle (.3); \node at (12,0) {\tiny 5};
\draw[] (16,0) circle (.3); \node at (16,0) {\tiny 4};
\draw[] (12,4) circle (.3); \node at (12,4) {\tiny 3}; 
\draw[] (16,4) circle (.3); \node at (16,4) {\tiny 2};
\draw[] (14,2) circle (.3); \node at (14,2) {\tiny 1};
\draw[line width=.01cm] [-](13.85,1.9)--(12,0.15);
\draw[line width=.01cm] [-](13.85,2.15)--(12.15,3.85);
\draw[line width=.01cm] [-](14.15,1.85)--(15.85,0.15);
\draw[line width=.01cm] [-](14.15,2.15)--(15.85,3.85);
\draw[line width=.01cm] [-](11.5,-0.5)--(11.5,4.5);
\draw[line width=.01cm] [-](11.5,4.5)--(16.5,4.5);
\draw[line width=.01cm] [-](16.5,4.5)--(16.5,-0.5);
\draw[line width =.01cm][-](16.5,-0.5)--(11.5,-0.5);
\node at (14, -1) {(c)};

\draw[] (18,0) circle (.3);\node at (18,0) {\tiny 1};
\draw[] (22,0) circle (.3);\node at (22,0) {\tiny 2};
\draw[] (18,4) circle (.3);\node at (18,4) {\tiny 3};
\draw[] (22,4) circle (.3);\node at (22,4) {\tiny 4};
\draw[] (20,2) circle (.3);\node at (20,2) {\tiny 5};
\draw[line width=.01cm] [-](17.5,-0.5)--(17.5,4.5);
\draw[line width=.01cm] [-](17.5,4.5)--(22.5,4.5);
\draw[line width=.01cm] [-](22.5,4.5)--(22.5,-0.5);
\draw[line width =.01cm][-](22.5,-0.5)--(17.5,-0.5);
\node at (20, -1) {(d)};
\end{tikzpicture}
\end{center}
\caption{Various 5-node user networks. (a)~fully connected (b)~circular (c)~star (d)~fully disconnected}
\label{Fig:1}
\end{figure}
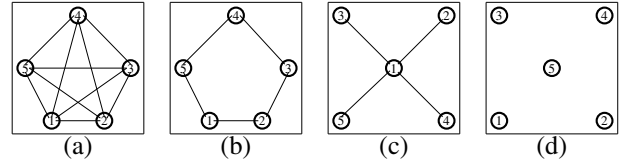
\section{The UCB-Network policy}
\label{UCB-Network}
Motivated by the well-known single agent policy UCB1~\cite{Auer}, we propose a distributed policy called the UCB-user. This is a deterministic policy, since, for a given action and reward history, the action chosen is deterministic. When each user in the network follows the UCB-user policy, we term the network policy as UCB-Network which is outlined in Algorithm~\ref{alg:UCB-Network}. 
\begin{algorithm}[htb]
   \caption{Upper-Confidence-Bound-Network (UCB-Network)}
   \label{alg:UCB-Network}
\begin{algorithmic}
   \STATE {Each user in $G$ follows UCB-user policy}  
   \STATE{\bfseries UCB-user policy for a user $v$:}
   \STATE{\bfseries Initialization:} For $1 \leq t \leq K$
   \STATE{- play arm $t$} 
   \STATE{\bfseries Loop:} For $K \leq t \leq n$
   \STATE - $a^v(t+1) = \underset{j}{\operatorname{argmax}} \, \, \hat{\mu}_{m_j^v(t)} + \sqrt{\frac{2 \ln t}{m_j^v(t)}}$
\end{algorithmic}
\end{algorithm}

The following theorem presents an upper bound on the expected regret of a generic network, under the UCB-Network policy.
\begin{theorem}
\label{Thm:3.1}
Assume that the network $G$ follows the UCB-Network policy to learn a stochastic MAB problem with $K$ arms. Further, assume that the rewards lie in $[0,1]$. Then,
\begin{itemize}
\item[(i)] The expected total regret of $G$ is upper bounded as:
\begin{equation*}
\label{eq:3.1}
\mathbb{E} \left[ R^G(n) \right] \leq \sum\limits_{i:\mu_i < \mu^*} \left[ C_G \frac{8 \ln n}{\Delta_i} + C_G \Delta_i \right] + b,
\end{equation*}
where $\Delta_i = \mu^* - \mu_i$, $\beta \in (0.25,1)$, $b~=~m \left( \frac{2}{4\beta -1} + \frac{2}{(4\beta-1)^2 \ln(1/\beta)} \right) \left( \sum\limits_{j=1}^K \Delta_j \right),$ and 
$C_G$ is a network dependent parameter, defined as follows.
\item[(ii)] Let $\gamma_k = \min \lbrace t \in \lbrace 1, \dots, n \rbrace : \vert \lbrace v \in V : m^v_i(t) \geq l_i = \frac{8 \ln n}{\Delta_i^2} \rbrace \vert \geq k \rbrace$ denote the smallest time index when at least $k$ nodes have access to at least $l_i$ samples of arm $i$. Let $\eta_k$ be the index of the `latest' node to acquire  $l_i$ samples of arm $i$ at $\gamma_k,$ such that $\eta_k \neq \eta_{k'}$ for $1 \leq k, k' \leq m$. Define $z_k = T_i(\gamma_k) := \left( T^1_i(\gamma_k), \dots, T^m_i(\gamma_k) \right)$, which contains the arm $i$ counts of all nodes at time $\gamma_k$. Then, $C_G l_i$ is the solution of the following optimisation problem:
\begin{equation}
\label{OptimizationProblem11}
\begin{aligned}
& \max \hspace{2mm} \Vert z_m \Vert_1 \\
& \text{s.t $\exists$ a sequence $\lbrace z_k \rbrace_{k=1}^m$} \\
& z_j(\eta_k) = z_k(\eta_k) \hspace{2mm} \forall j \geq k\\
& \langle z_k, A(\eta_k,:) \rangle \geq l_i, \hspace{2mm} 1\leq k \leq m
\end{aligned}
\end{equation}
\end{itemize}
\end{theorem}
\vspace{3mm}
\begin{proof}
Refer Appendix A.
\end{proof}

\emph{Interpretation of \eqref{OptimizationProblem11}}: Under the UCB-Network policy, suppose a node has acquired at least $l_i$ samples of a sub-optimal arm $i$. As shown in the Lemma~\ref{Thm:A.1} in the Appendix A that such a node will not play the sub-optimal arm $i$ subsequently with high probability. Next, note that, $z_k$ is a vector of arm $i$ counts (self plays) of all nodes at time $\gamma_k$. The objective function in \eqref{OptimizationProblem11} represents the sum of arm $i$ counts of all nodes at the smallest time index, when all nodes have access to at least $l_i$ samples of arm $i$. The solution to~\eqref{OptimizationProblem11} represents the maximum number of samples of arm $i$ required by the entire network such that~ $(a)$ Each node has access to at least $l_i$ samples of arm $i$ (the last constraint in~\eqref{OptimizationProblem11}), and $(b)$~Each node stops choosing arm $i$ after it has access to $l_i$ samples of it (the penultimate constraint in~\eqref{OptimizationProblem11}). 

For example, the solution to \eqref{OptimizationProblem11} for an $m$-node star network (shown in Fig.~\ref{Fig:1}) is $(m-1)l_i$. This corresponds to the scenario where the center node never chooses the sub-optimal arm $i$, and each leaf node chooses it $l_i$ times.  

\emph{Proof sketch:} First, we show that any node $v$ plays any sub-optimal arm $i$ in a given round $t$ with small probability after it has $l_i$ samples of it, in Lemma~\ref{Thm:A.1}. Using Lemma~\ref{Thm:A.1}, we  then upper bound the expected regret of the network after each node has $l_i$ samples of the sub-optimal arm $i$. Next, we upper bound the maximum number of samples of the sub-optimal arm $i$ required by the entire network such that each node has access to $l_i$ samples of it, in Lemma~\ref{Thm:A.2}. Finally, we obtain the desired upper bound by combining Lemma ~\ref{Thm:A.1} and Lemma ~\ref{Thm:A.2}. A detailed proof, along with Lemma~\ref{Thm:A.1} and ~\ref{Thm:A.2} is given in the Appendix A.
\subsection{Application to typical networks}
Solving \eqref{OptimizationProblem11} for an arbitrary network is analytically complex.  Hence, we solve the problem for a few specific  networks that range from high connectivity to low connectivity; namely, the $m$-node Fully Connected (FC), circular, star and Fully Disconnected (FD) networks. For $m = 5$, these networks are shown in Fig. 1. It is easy to verify that the solution to \eqref{OptimizationProblem11} for these four networks are $l_i$, $(m-1)l_i$, $\lfloor \frac{m}{2} \rfloor l_i$ and $ml_i,$ respectively. We can then evaluate the upper bounds in Theorem \ref{Thm:3.1}. 
\\
\textbf{Corollary 1}
For an $m$-node FC network:
\begin{equation}
\mathbb{E}[R^G(n)] \leq \sum\limits_{i:\mu_i < \mu^*} \left( \frac{8 \ln n}{\Delta_i} + \Delta_i \right) + b.
\end{equation} 
\textbf{Corollary 2}
For an $m$-node circular network:
\begin{equation}
\mathbb{E}[R^G(n)] \leq \Big\lfloor \frac{m}{2} \Big \rfloor \sum\limits_{i:\mu_i < \mu^*} \left( \frac{8 \ln n}{\Delta_i} + \Delta_i \right) + b.
\end{equation}
\textbf{Corollary 3}
For an $m$-node star network:
\begin{equation}
\mathbb{E}[R^G(n)] \leq (m-1) \sum\limits_{i:\mu_i < \mu^*} \left( \frac{8 \ln n}{\Delta_i} + \Delta_i \right)  + b.
\end{equation} 
\textbf{Corollary 4}
For an $m$-node FD network:
\begin{equation}
\mathbb{E}[R^G(n)] \leq m \sum\limits_{i:\mu_i < \mu^*} \left( \frac{8 \ln n}{\Delta_i} + \Delta_i \right) + b.
\end{equation} 

A key insight can be obtained from the above corollaries is that, the expected regret of a network decreases by a factor of $m$, $2$ and $m/(m-1)$ in the cases of $m$-node FC, circular and star networks respectively, compared to FD network.  
\section{Lower bounds on the expected regret}
\label{Lowerbound1}
In this section, we derive lower bounds on the expected regret of the network under various classes of policies. Our first lower bound is a universal bound which is independent of the user network, and holds for large class of `reasonable' learning policies. Second, we derive a network-dependent lower bound for a class of \emph{Non-Altruistic and Individually Consistent} (NAIC) policies -- a class that includes network extensions of well-studied policies like UCB1 and Thompson sampling. Finally, we derive a refined lower bound for large star networks under NAIC policies.

Throughout this section, we assume that the distribution of each arm is parametrised by a single parameter. We use $\boldsymbol{\theta} = \left( \theta_1, \dots, \theta_K \right) \in \Theta^K = \boldsymbol{\Theta} $ to denote the parameters of arms $1$ to $K$ respectively. 
Suppose $f(x;\theta_j)$ be the reward distribution for arm $j$ with parameter $\theta_j$.  Let $\mu(\theta_j)$ be the mean of arm $j$, and $\theta^* = \underset{1 \leq j \leq K}{\arg\max}~\mu(\theta_j)$. Define the parameter sets for an arm $j$ as   
\begin{equation*} 
\begin{aligned}
\boldsymbol{\Theta_j} &= \lbrace \boldsymbol{\theta} : \mu(\theta_j) < \max_{i \neq j} \mu (\theta_i) \rbrace \\
\boldsymbol{\Theta_j^*} &= \lbrace \boldsymbol{\theta} : \mu(\theta_j) > \max_{i \neq j} \mu (\theta_i) \rbrace
\end{aligned}  
\end{equation*}    

Note that $\boldsymbol{\Theta_j}$ contains all parameter vectors in which the arm $j$ is a sub-optimal arm, and $\boldsymbol{\Theta_j^*}$ contains all parameter vectors in which the arm $j$ is the optimal arm. Let $kl(\beta || \lambda)$ be the KL divergence of the distribution parametrised by $\lambda$, from the distribution parametrised by $\beta$.

\emph{[A1]} We assume that the set $\Theta$ and $kl(\beta || \lambda)$ satisfy the following~\cite{Lai}:
\begin{itemize}
\item[(i)] $f(.;.)$ is such that $0 < kl(\beta || \lambda) < \infty$ whenever $\mu(\lambda) > \mu(\beta).$ 
\item[(ii)] $\forall \epsilon >0$ and $\forall \beta, \lambda$ such that $\mu(\lambda) > \mu(\beta), \exists \delta = \delta(\epsilon,\beta,\lambda) > 0$ for which $| kl(\beta || \lambda) - kl(\beta || \lambda')| <\epsilon$ whenever $\mu(\lambda) \leq \mu(\lambda') \leq \mu(\lambda) + \delta.$
\item[(iii)] $\Theta$ is such that $\forall \lambda \in \Theta$ and $\forall \delta > 0, \exists \lambda' \in \Theta$ such that $\mu(\lambda) < \mu(\lambda') < \mu(\lambda) + \delta.$
\end{itemize}  
\begin{theorem}
\label{Thm:4.4}
Let $G$ be an $m$-node connected generic network, and suppose \emph{[A1]} holds. Consider the set of policies for users in $G$ to learn a $K$-arm stochastic MAB problem with a parameter vector of arms as $\boldsymbol{\theta} \in \boldsymbol{\Theta}$ such that $\mathbb{E}_{\boldsymbol{\theta}}[m_j^G(n)] = o(n^c)$ $\forall$ $c>0$, for any sub-optimal arm $j$. Then, for $\delta \in (0,1)$, the following holds.
\begin{equation}
\label{eq:UniversalLowerbound}
\liminf \limits_{n \rightarrow \infty}\frac{\mathbb{E}_{\boldsymbol{\theta}} [m_j^G(n)]}{\ln n} \geq  \frac{1-\delta}{1+\delta} \cdot \frac{1}{kl(\theta_j || \theta^*)}. 
\end{equation}
\end{theorem}
\vspace{3mm}
\begin{proof}
Refer Appendix B.
\end{proof}

Note that the above universal lower bound is based on fundamental statistical limitations, and is independent of the network $G$. Next, we define the class of NAIC policies, and derive a network-dependent lower bound for this class. In the rest of this section, we assume that each arm is associated with a discrete reward distribution, which assigns a non-zero probability to each possible value. 

Let $\omega$ be a sample path, which consists of all pairs of actions and the corresponding rewards of all nodes from rounds $1$ through $n$: 
\begin{equation*}
\omega = \lbrace  (a^v(t), X^v_{a^v(t)}(t)) : v \in V, 1 \leq t \leq n \rbrace.  
\end{equation*}
Also, define
\begin{align*}
\omega_v & = \lbrace (a^u(t), X^u_{a^u(t)}(t)) : u \in \mathcal{N}(v), 1 \leq t \leq n \rbrace \\ \omega_{\bar{v}} &= \lbrace (a^u(t), X^u_{a^u(t)}(t)) : u \in \mathcal{N}(v)^c, 1 \leq t \leq n \rbrace.
\end{align*}

\textbf{Definition 1} [Individually consistent policy] A policy followed by a user $v$ is said to be \emph{\network-oblivious} if, for any sub-optimal arm $i$, and for any policy of a user $u \in \mathcal{N}(v)\setminus \lbrace v \rbrace$ 
\begin{equation}
\label{eq:def1}
\mathbb{E}[T^v_i(n) \vert \omega_{\bar{v}}] = o(n^a), \,\, \forall \, a >0, \,\, \forall \, \omega_{\bar{v}}.
\end{equation}

\textbf{Definition 2} [Non-altruistic policy] A policy followed by a
user $v$ is said to be \emph{non-altruistic} if there exist $a_1,
a_2$, not depending on time horizon $n$, such that the following
holds. For any $n$ and any sub-optimal arm $i$, the expected number of
times that the policy {\em plays} arm $i$ after having {\em obtained}
$a_1 \ln n$ samples of that arm is no more than $a_2$, irrespective
of the policies followed by the other users in the network.

It can be shown that UCB-user and Thompson sampling~\cite{Shipra} are NAIC policies. In particular, we show that the UCB-user policy is an NAIC policy in Lemma~\ref{Thm:NAIC} in Appendix~A.

\emph{Example of a policy which is not \network-oblivious~:} Consider a 2-armed stochastic bandit problem with Bernoulli rewards with means $\mu_1, \mu_2$, where $\mu_1 > \mu_2$. Consider the 3-node line graph with node 2 as the center node. Let the policy followed by node 1 be as follows: $a^1(t) = a^2(t-1)$ for $t>1$ and $a^1(1) = 2$ (we call this policy \textit{follow node $2$}). Consider the following $\omega_{\bar{1}} = \lbrace (a^3(t) = 2, X^3_2(t) = 0) : 1 \leq t \leq n \rbrace$. Then, $\mathbb{E} [T^1_2(n) | \omega_{\bar{1}} ] = n$ under the node $2$'s policy as \textit{follow node $3$}, which clearly violates the equation \eqref{eq:def1}. Hence, the \emph{follow node 2} policy for node $1$ is not \network-oblivious. 

Note that the above policy, \textit{follow node $u$}, is in fact a non-trivial and rather well-performing policy that we will revisit in Section~\ref{FYL}. We now derive a network-dependent lower bound for the class of NAIC policies

\begin{theorem}
\label{Thm:4.1}
Let $G=(V, E)$ be a network with $m$ nodes, and suppose [A1] holds. If each node in $V$ follows an NAIC class policy to learn a $K$-arm stochastic MAB problem with a parameter vector of arms as $\boldsymbol{\theta} = (\theta_1, \dots, \theta_K) \in \boldsymbol{\Theta_j}$,  and $\delta \in (0,1)$ then, the following lower bounds hold:
\begin{align}
(i) & \liminf \limits_{n \rightarrow \infty} \frac{\mathbb{E}_{\boldsymbol{\theta}} [m_j^v(n) \vert \omega_{\bar{v}} ]}{\ln n} \geq \frac{1-\delta}{1 + \delta}\cdot \frac{1}{kl(\theta_j || \theta^*)}, \forall v \in V \nonumber \\ 
& \liminf \limits_{n \rightarrow \infty} \frac{\mathbb{E}_{\boldsymbol{\theta}} [m_j^v(n) ]}{\ln n} \geq \frac{1-\delta}{1 + \delta} \cdot \frac{1}{kl(\theta_j || \theta^*)}, \forall v \in V \nonumber \\ 
(ii) & \liminf\limits_{n \rightarrow \infty} \frac{\mathbb{E}_{\boldsymbol{\theta}} [m_j^G(n)]}{\ln n} \geq L_G \cdot \frac{1-\delta}{1 + \delta} \cdot \frac{1}{kl(\theta_j || \theta^*)}, \label{LB:Inconsistent}
\end{align}
where $L_G$ can be obtained from the solution to the following optimisation problem: 
\begin{equation}
\label{OptimizationProblem2}
\begin{aligned}
& \min \hspace{2mm} \Vert z_m \Vert_1 \\
& \text{s.t $\exists$ a sequence $\lbrace z_k \rbrace_{k=1}^m$} \\
& z_i(\eta_k) = z_k(\eta_k) \hspace{2mm} \forall i \geq k \\
& \langle z_k, A(n_k,:) \rangle \geq q_j = \frac{1-\delta}{1 + \delta} \cdot \frac{\ln n}{kl(\theta_j|| \theta^*)}, \hspace{2mm} \forall k.
\end{aligned}
\end{equation}
\end{theorem}
\vspace{3mm}
\begin{proof}
Refer Appendix C.
\end{proof}
The notation used in~\eqref{OptimizationProblem2} is the same as the notation in Theorem~\ref{Thm:3.1}, except that $l_i$ is replaced with $q_j$. Further, $L_G$ is obtained by dividing the solution to~\eqref{OptimizationProblem2} by $q_j$. Similar to \eqref{OptimizationProblem11}, solving (\ref{OptimizationProblem2}) analytically for an arbitrary network is difficult. Hence, we focus on solving (\ref{OptimizationProblem2}) for the networks shown in Fig. \ref{Fig:1}, and provide the corresponding lower bounds below. Let $\Delta_i = \mu(\theta^*) - \mu(\theta_i)$.
\\
\textbf{Corollary 5}
For an $m$-node FC network:
\begin{equation}
\liminf\limits_{n \rightarrow \infty} \frac{\mathbb{E}_{\boldsymbol{\theta}}[R^G(n)]}{\ln n} \geq \sum\limits_{i : \Delta_i > 0} \frac{1-\delta}{1+\delta} \cdot \frac{\Delta_i}{kl(\theta_i || \theta^*)}.
\end{equation} 
\textbf{Corollary 6} 
For an $m$-node circular network:
\begin{equation}
\liminf\limits_{n \rightarrow \infty} \frac{\mathbb{E}_{\boldsymbol{\theta}}[R^G(n)]}{\ln n} \geq \frac{m}{3}\sum\limits_{ i : \Delta_i > 0} \frac{1-\delta}{1+\delta} \cdot \frac{\Delta_i}{kl(\theta_i || \theta^*)}.
\end{equation} 
\textbf{Corollary 7} 
For an $m$-node star network:
\begin{equation}
\label{eq:LBStar}
\liminf\limits_{n \rightarrow \infty} \frac{\mathbb{E}_{\boldsymbol{\theta}}[R^G(n)]}{\ln n} \geq \sum\limits_{i : \Delta_i > 0} \frac{1-\delta}{1+\delta} \cdot \frac{\Delta_i}{kl(\theta_i || \theta^*)}.
\end{equation} 
\textbf{Corollary 8} 
For an $m$-node FD network:
\begin{equation}
\liminf\limits_{n \rightarrow \infty} \frac{\mathbb{E}_{\boldsymbol{\theta}}[R^G(n)]}{\ln n} \geq m  \sum\limits_{i : \Delta_i > 0} \frac{1-\delta}{1+\delta} \cdot \frac{\Delta_i}{kl(\theta_i || \theta^*)}.
\end{equation}  

From corollaries 1-8, we infer that the upper bound of the UCB-Network policy and the lower bound given by~\eqref{LB:Inconsistent} are of the same order, for FC $(\ln n)$, circular $(m \ln n)$ and FD $(m \ln n)$ networks. However, for star networks, there is a large gap between the UCB-Network upper bound and the lower bound for NAIC policies in \eqref{eq:LBStar}. Since the UCB-Network is an NAIC class policy, we proceed to ascertain if either of these bounds is too loose for star networks. Our special interest in star networks is due to the prevalence of hubs in many social networks, and as we shall see in the next section, this hierarchical structure can be exploited to enhance the learning rate.

Next, we consider a specific instance of a large star network, for which we derive a refined lower bound for the class of NAIC policies. This refined lower bound is of the same order as the regret upper bound for the UCB-Network policy, implying that the upper bound in Theorem~\ref{Thm:3.1} is tight in an order sense, and cannot be improved in general. 
\begin{theorem}
\label{Thm:4.3}
Let $G_n=(V_n, E_n)$ be a sequence of $m_n$-node star networks learning a $2$-arm stochastic MAB problem with mean rewards $\mu_a, \mu_b$ such that $\mu_a > \mu_b$. Suppose $m_n \geq 2 \cdot \frac{\ln n}{ kl(\mu_b || \mu_a) }$, and that each node follows an NAIC policy. Then,
\begin{equation}
\liminf\limits_{n \rightarrow \infty} \frac{\mathbb{E} [m_2^{G_n}(n)]}{(m_n-1) \ln n} \geq  \frac{1}{kl(\mu_b || \mu_a)}. 
\end{equation}
\end{theorem}
\vspace{3mm}  
\begin{proof}
Refer Appendix D.
\end{proof}

We now briefly explain the intuition behind Theorem~\ref{Thm:4.3}. In a large star network, the center node learns the sub-optimal arm very quickly (in a few rounds), since it has access to a large number of samples in each round. Under an NAIC policy, once a node has enough samples to learn that an arm is sub-optimal, by definition, it stops choosing that arm with high probability. Hence, the center node stops choosing the sub-optimal arm with high probability, which in turn ensures that the leaf nodes learn the sub-optimal arm themselves, by choosing the sub-optimal arm $O(\ln n)$ times. This leads to a regret of $O((m-1) \ln n)$. Our simulation results, in Table~\ref{Tab:1}, also illustrates this behaviour, for the UCB-Network policy (which is NAIC) on large star networks.       

Theorem~\ref{Thm:4.3} asserts that, for a fixed, large time horizon $n$, we can construct a large star network with $m$ nodes, whose expected regret is atleast  $O((m-1) \ln n)$. This lower bound matches with the upper bound for UCB-Network in Theorem~\ref{Thm:3.1}. Thus, we conclude that the class of NAIC policies could suffer a large regret, matching the upper bound in an order sense. However, for the same star network and time horizon, the universal lower bound in~\eqref{eq:UniversalLowerbound} turns out to be $O(\ln n)$.  This gap suggests the possibility that there might exist good learning policies (which are not NAIC) for a star network, with regret matching the universal lower bound. In the next section, we propose one such policy, which does not belong to the NAIC class.  
\section{The Follow Your Leader (FYL) Policy}
\label{FYL}
In this section, we first outline a policy called Follow Your Leader (FYL) for a generic $m$-node network. The policy is based on exploiting high-degree hubs in the graph; for this purpose, we define the dominating set and the dominating set partition.

\textbf{Definition 3} [Dominating set of a graph]~\cite{haynes1998fundamentals} A \textit{dominating set} $D$ of a graph $G = (V,E)$ is a subset of $V$ such that every node in $V\setminus D$ is adjacent to atleast one of the nodes in $D$. The cardinality of the smallest dominating set of $G$ is called as the \emph{domination number}. 

\textbf{Definition 4} [Dominating set partition of a graph] 
Let $D$ be a dominating set of $G$. A dominating set partition based on $D$ is obtained by partitioning $V$ into $|D|$ components such that each component contains a node in $D$ and a subset of its one hop neighbors. 

Note that given a dominating set for a graph, it is easy to obtain a corresponding dominating set partition. The FYL policy for an $m$-node generic network is outlined in Algorithm~\ref{alg:FYL}. Under the FYL policy, all nodes in the dominating set are called {\em leaders} and all other nodes as {\em followers}; the follower nodes follow their leaders while choosing an action in a round.     
As we argued in Section~\ref{Lowerbound1}, the policy deployed by a follower node in FYL is not \network-oblivious. The following theorem presents an upper bound on the expected regret of an $m$-node star network which employs the FYL policy.
\begin{algorithm}[htb]
   \caption{Follow Your Leader (FYL) Policy}
   \label{alg:FYL}
   \begin{algorithmic}
   \STATE {\bfseries Input:} Graph $G$, a dominating set $D$ and a dominating set partition 
   \STATE{\bfseries Leader - Each node in $D:$ }
   \STATE{Follows the UCB-user policy by using the samples of itself and its one-hop neighbours in the same component}
   \STATE{\bfseries Follower - Each node in $V\setminus D:$ }
   \STATE{In round $t=1:$} 
   \STATE{- Chooses an action randomly from $\mathcal{K}$}
   \STATE{In round $t > 1$}
   \STATE{-  Chooses the action taken by the leader in its component, in the previous round $(t-1)$}
\end{algorithmic}
\end{algorithm}

\begin{theorem}[FYL regret bound, star networks]
\label{Thm:5.1}
Suppose the star network $G$ with a dominating set as the center node, follows the FYL policy to learn a stochastic MAB problem with $K$ arms. Assume that the rewards lie in $[0,1]$. Then, 
\begin{align*}
\mathbb{E}[R^G(n)] \leq \sum\limits_{i:\mu_i < \mu^*}^K \frac{8 \ln n}{\Delta_i} + d,
\end{align*}
where $d = \Big[ 2m - 1 + \frac{2 m}{4 \beta -1} \left( 1 + \frac{1}{(4 \beta -1) \ln (1/ \beta)} \right) \Big] \sum\limits_{j=1}^K \Delta_j$, $\Delta_i= \mu^* - \mu_i$ and $\beta \in (0.25,1)$.
\end{theorem}
\vspace{3mm}
\begin{proof}
Refer Appendix E.
\end{proof} 
A key insight obtained from Theorem~\ref{Thm:5.1} is that an $m$-node star network with the FYL policy incurs an expected regret that is lower by a factor $(m-1)$, as compared to any NAIC policy. More importantly, we  observe that the regret upper bound under the FYL policy meets the universal lower bound in~\eqref{eq:UniversalLowerbound}. Hence, we conclude that the FYL policy is order optimal for star networks. 

Finally, we present a result  that asserts an upper bound on the expected regret of a generic network under the FYL policy. 
\begin{theorem}[FYL regret bound, general networks]
\label{Thm:5.2}
Let $D$ be a dominating set of an $m-$node network $G= (V,E).$ Suppose $G$ with the dominating set $D$ employs the FYL policy to learn a stochastic MAB problem with $K$ arms, and the rewards lie in $[0,1]$, then
\begin{equation*}
\mathbb{E}[R^G(n)] \leq \sum\limits_{i:\mu_i < \mu^*}^K \frac{8|D| \ln n}{\Delta_i} + d.
\end{equation*}
\end{theorem} 
\vspace{3mm}
\begin{proof}
Refer Appendix F.
\end{proof}

From the above theorem we infer that, the expected regret of a network scales linearly with the cardinality of a given dominating set. Hence, in order to obtain a tighter upper bound, we need to supply a smallest dominating set $D^*$  to the FYL policy. Suppose, if we provide $D^*$ as the input to the FYL policy, then we obtain an improvement of factor $m/|D^*|$ in the expected regret of an $m$-node network compared to the fully disconnected network.

It is known that, computing a smallest dominating set of a given graph is an NP-hard problem~\cite{Kuhn}. However, fast distributed approximation algorithms for the same are well-known in the literature. For example, Algorithm $35$ in \cite{Kuhn} finds a smallest dominating set with an approximation factor $\log(\text{MaxDegree}(G)).$  Also, upper bounds on the domination number for specific networks such as Erdos-Renyi, power-law preferential attachment and random geometric graphs are available in~\cite{wieland2001domination,molnar2014dominating,bonato2015domination}. 

\section{Numerical Results}
\label{Results}
\begin{figure*}[!htb]
\minipage{0.32\textwidth}
  \includegraphics[scale=0.21]{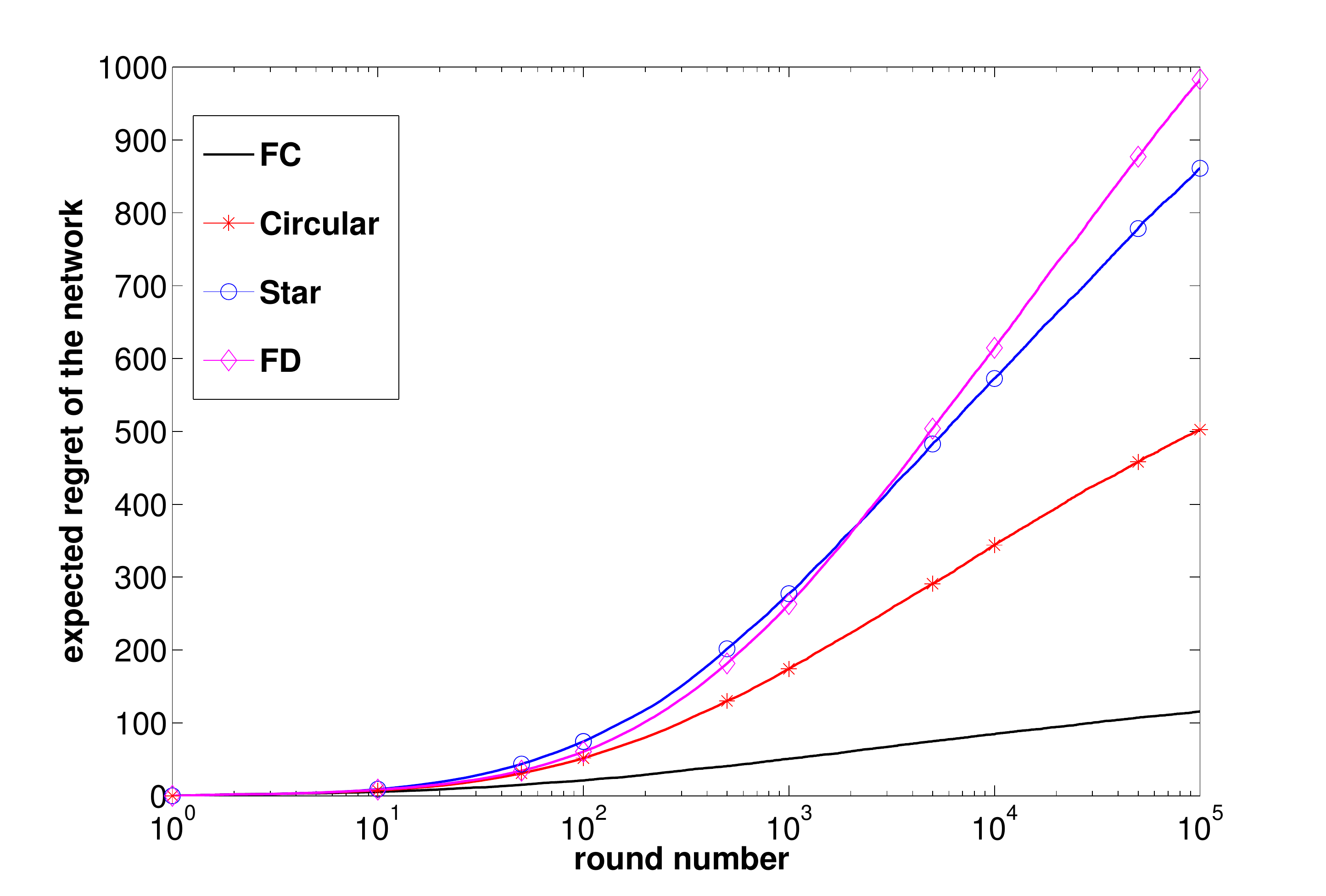} 
\endminipage\hfill
\minipage{0.32\textwidth}
  \includegraphics[scale=0.21]{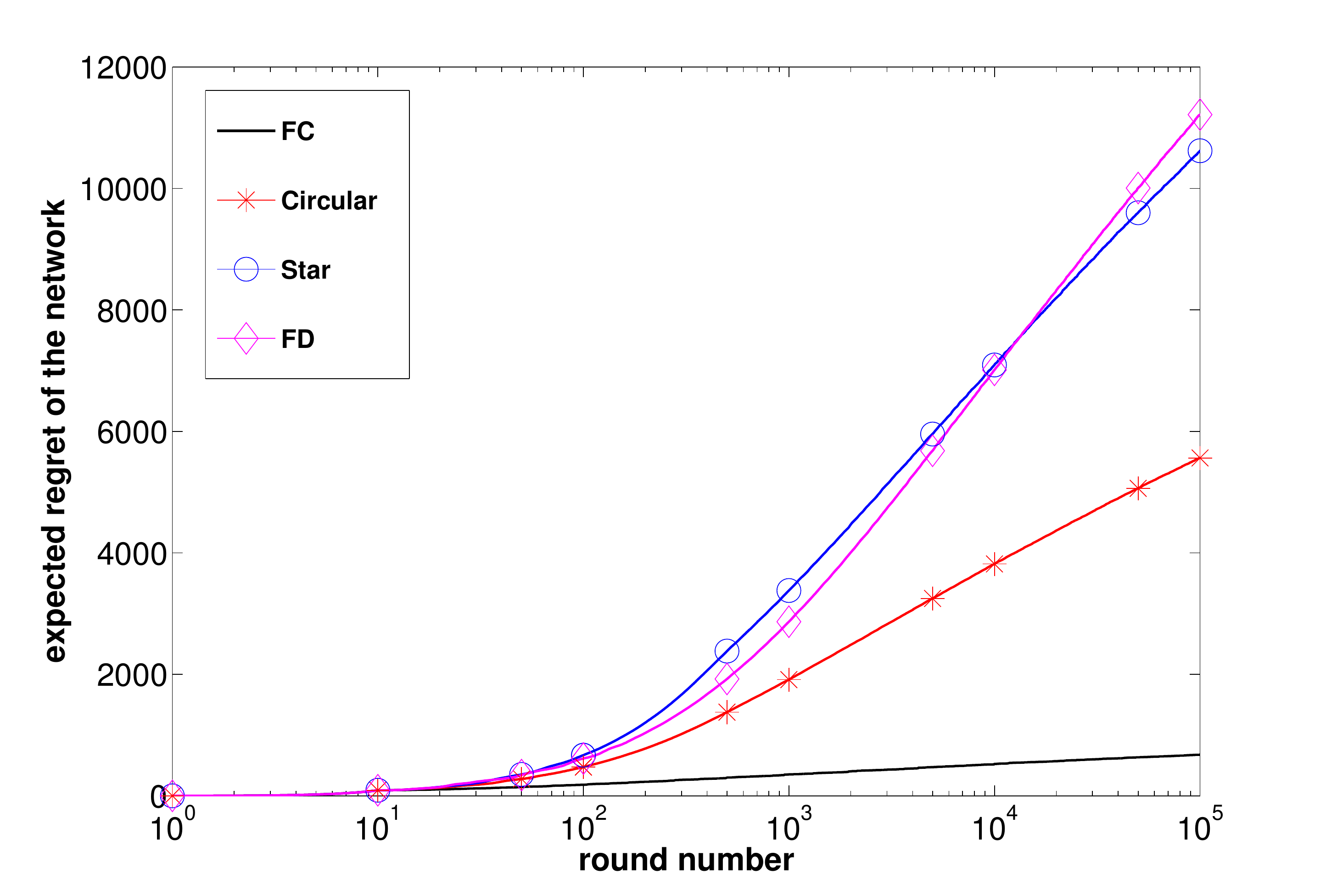} 
\endminipage\hfill
\minipage{0.32\textwidth}
 \includegraphics[scale=0.205]{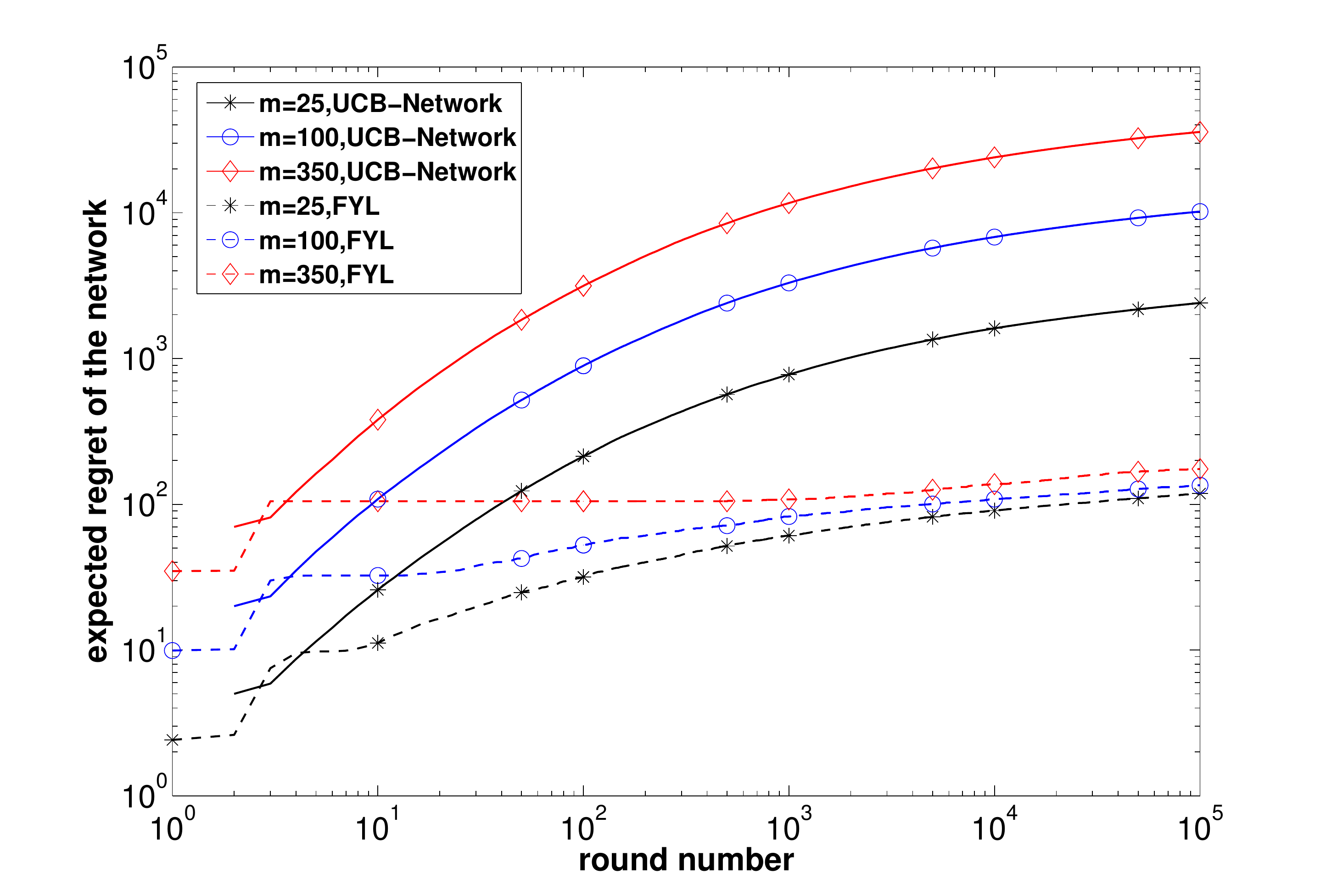} 
\endminipage
\end{figure*}
\begin{figure*}[!htb]
\minipage{0.32\textwidth}
  \includegraphics[scale=0.21]{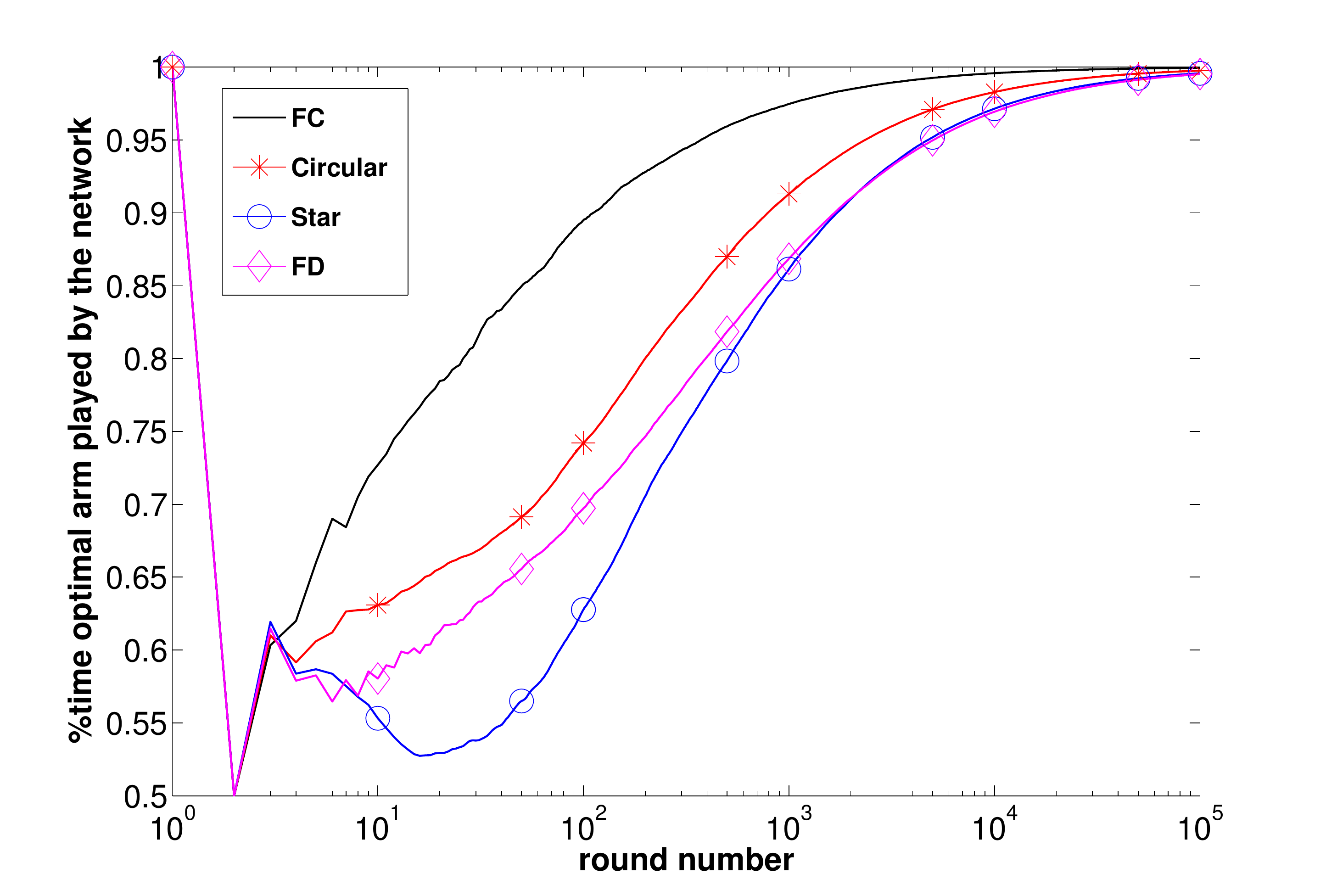}   
  \caption{Performance comparison of UCB-Network policy on various 10 node networks: 2 arms, Bernoulli rewards with means 0.7 and 0.5}
  \label{Fig:2}
\endminipage\hfill
\minipage{0.32\textwidth}
 \includegraphics[scale=0.21]{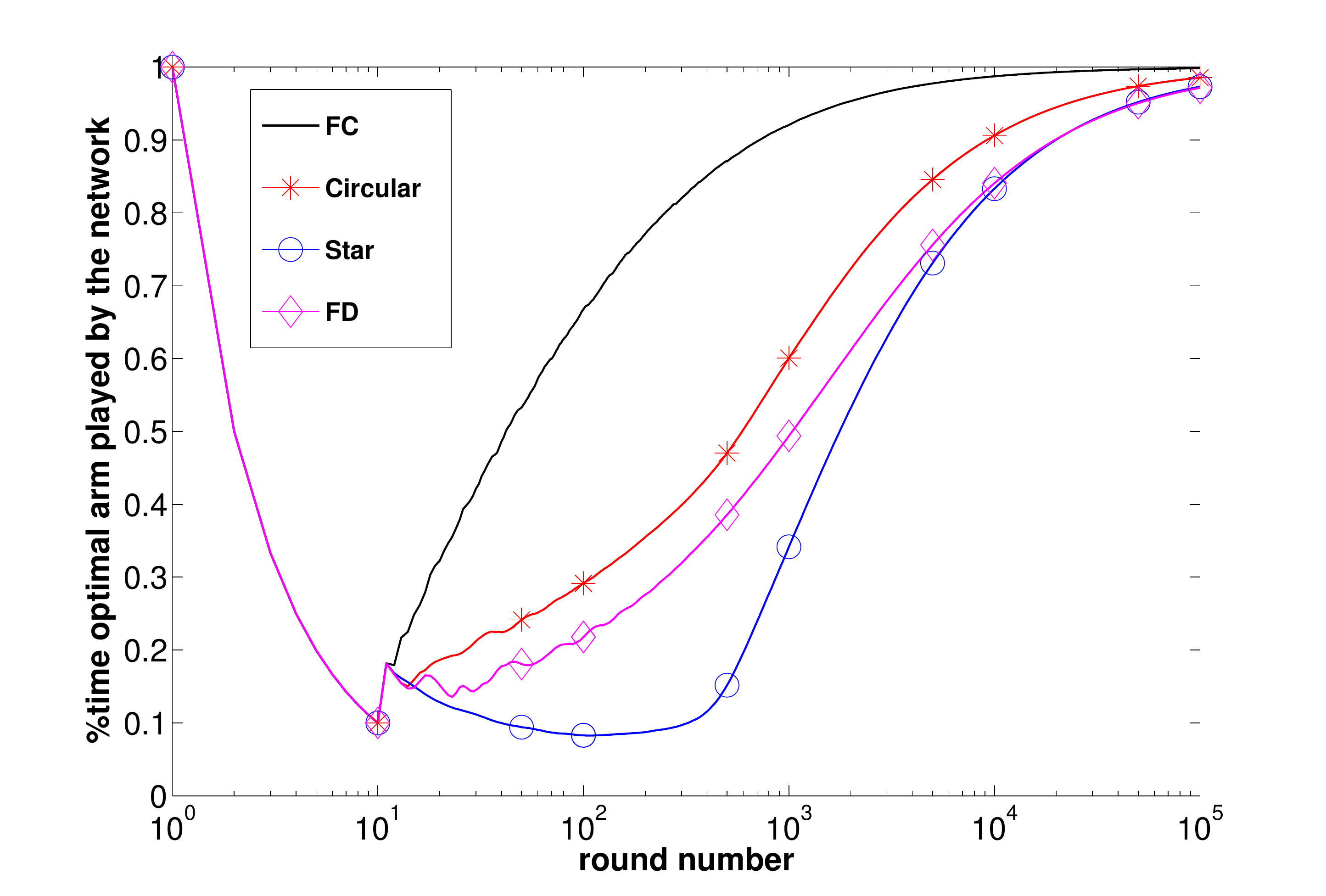}  
  \caption{Performance comparison of UCB-Network policy on various 20 node networks: 10 arms, Bernoulli rewards with means $1, 0.9, \dots, 0.1$}
  \label{Fig:3}
\endminipage\hfill
\minipage{0.32\textwidth}
 \includegraphics[scale=0.235]{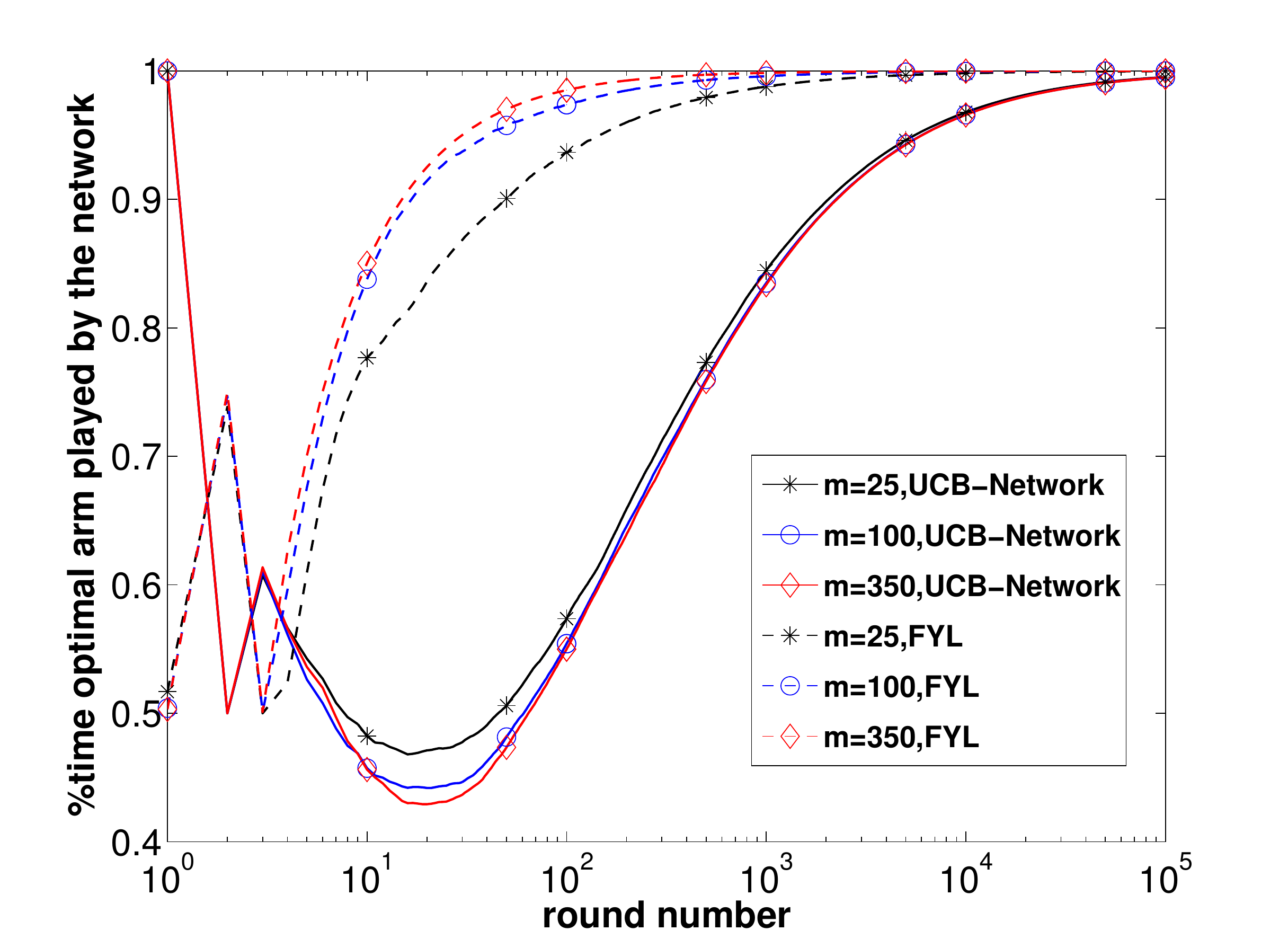} 
 \caption{Performance comparison of UCB-Network and FYL policies on various star networks: 2 arms, Bernoulli rewards with means 0.7 and 0.5}
  \label{Fig:4}
\endminipage
\end{figure*}
We now present some simulations that serve to corroborate our analysis. The simulations have been carried out using MATLAB, and are averaged over 100 sample paths.  We fix the time horizon $n$ to be $10^5$. 
\subsection{Performance of UCB-Network on various networks} 
We consider the following two scenarios: $(i)$ 10 node FC, circular, star and FD networks, 2 arms, Bernoulli rewards with means $0.7, 0.5$, and $(ii)$ 20 node FC, circular, star and FD networks, 10 arms, Bernoulli rewards with means $1, 0.9, 0.8, \dots, 0.1$. We run the UCB-Network policy for these scenarios, and calculate the expected regret of the network and percentage of time the optimal arm is played by the network. The results are shown in Fig.~\ref{Fig:2} and~\ref{Fig:3}. It can be observed from Fig.~\ref{Fig:2} and~\ref{Fig:3} that the expected regret of the network decreases and the percentage of time the optimal arm is chosen by the network increases, as connectivity of the network increases. This is because, an increase in the connectivity of the network increases the number of observations available to a user, in a given round.
\subsection{Performance of UCB-Network on star networks}
We consider 5, 10, 25, 50, 100, 200 and 350 node star networks, each learning a 2-armed stochastic bandit problem with Bernoulli rewards of means 0.7 and 0.5. We run the UCB-Network policy on the aforementioned networks, and summarise the results in Table~\ref{Tab:1}. 
Observe that, the expected number of times the center node chooses arm 2 (sub-optimal arm) decreases as the network size increases. This forces each leaf node to choose arm 2 on its own in order to learn. Therefore, as the star network size increases, the expected regret of the network can be approximated as the product of the network size and the expected regret of an isolated node.
\begin{table}
\centering
\vspace{0.2cm}
\caption{Expected number of times arm 2 played by a node in star networks under UCB-Network policy, 2 armed MAB problem with Bernoulli mean rewards as 0.7 and 0.5}
\label{Tab:1}
\begin{tabular}{|c|c|l|} \hline
Size of the network&Center Node&Leaf Node\\ \hline
5 & 66 & 448 \\ \hline
10 & 79 & 442 \\ \hline
25 & 33 & 486 \\ \hline
50 & 10 & 502\\ \hline
100 & 1 & 514\\ \hline
200 & 1 & 516\\ \hline
350 &1  & 513\\ \hline
\end{tabular}
\end{table}
\subsection{Comparison of UCB-Network and FYL policies}
We consider 25, 100 and 350 node star networks learning a 2-arm stochastic bandit problem with Bernoulli rewards of means 0.7 and 0.5. We run both UCB-Network and FYL policies on the above-mentioned networks. It can be observed from Fig.~\ref{Fig:4} that the star networks incur much smaller expected regret under the FYL policy, as compared to UCB-Network, and learn the optimal arm much faster. 
\section{Concluding Remarks}
\label{Conclusions}
We studied the collaborative learning of a stochastic MAB problem by a group of users connected through a social network. We analysed the regret performance of widely-studied   single-agent learning policies, extended to a network setting. Specifically, we showed that the class of NAIC policies (such as UCB-Network) could suffer a large expected regret in the network setting. We then proposed and analysed the FYL policy, and demonstrated that exploiting the structure of the network leads to a substantially lower expected regret. In particular, the FYL policy's upper bound on the expected regret matches the universal lower bound, for star networks, proving that the FYL policy is order optimal. This also suggests that using the center node as an information hub is the right information structure to exploit.   

In terms of future research directions, we plan to study this model for other flavours of MAB problems such as linear stochastic~\cite{Abbasi} and contextual bandits~\cite{Li}. Even in the basic stochastic bandit model considered here, several fundamental questions remain unanswered. For a given network structure, what is the least regret achievable by {\em any} local information-constrained learning strategy? Is it possible in a general network to outperform `good single-agent' policies (i.e., those that work well individually, like UCB) run independently throughout the network? If so, what kind of information sharing/exchange might an optimal strategy perform? 
It is conceivable that there could be sophisticated distributed bandit strategies that could signal within the network using their action/reward sequences, which in turns begs for an approach relying on information-theoretic tools. 
 
\bibliography{example_paper}
\bibliographystyle{IEEEtran}
\section*{Appendix A}
\label{UCB-Network Generic Upper Bound}
We require the following Lemma~\ref{Lemma:coupling}, \ref{Thm:A.1}, \ref{Thm:A.2} and inequality to prove Theorem~\ref{Thm:3.1}.
\\
\emph{Hoeffding's Maximal Inequality} \cite{Bubeck2}: Let $X_1, X_2, \dots$ be centered i.i.d random variables lying in $[0,1]$. Then, for any $x >0$ and $t \geq 1$,
\begin{equation*}
\mathbb{P}\left( \exists s \in \lbrace 1, \dots, t \rbrace, \sum\limits_{i=1}^s X_i > x \right) \leq \exp \left( -\frac{2x^2}{t} \right).
\end{equation*}  
In order to introduce Lemma~\ref{Lemma:coupling}, we need the following.

Consider a new probability space with probability measure $\tilde{\mathbb{P}}$, for the rewards corresponding to all arms. First, for a fixed node $v \in V$, for each action $i \in \mathcal{K}$, we consider a sequence of i.i.d. random variables $\lbrace Y_i(k) \rbrace_{k=1}^\infty$ with arm $i$'s distribution. If a node $v$ or its neighbours choose an arm $i$, then they receive the rewards from the sequence $\lbrace Y_i(k) \rbrace_{k=1}^\infty$. Next, for each $u \in V \setminus \mathcal{N}(v)$, for each action $i \in \mathcal{K}$, we consider a sequence of i.i.d. random variables $\lbrace X^u_i(k) \rbrace_{k=1}^\infty$ with arm $i$'s distribution. If a node $u \in V \setminus \mathcal{N}(v)$ chooses an arm $i$, then it receives a reward from the sequence $\lbrace X^u_i(k) \rbrace_{k=1}^\infty$. Recall that, in the setting described in Section~\ref{Model}, if a user $v$ chooses arm $i$, then it receives a reward from the sequence $\lbrace X^u_i(k) \rbrace_{k=1}^\infty$. In this probability space, we considered the probability measure to be $\mathbb{P}$. 

We prove that the probability of a sample path of the network in both probability spaces are equal, in the following lemma. Hence, this allows us to equivalently work in the new probability space, as and when appropriate.
\begin{lemma}
\label{Lemma:coupling}
Consider an $m$-node undirected user graph. Let $A(t)$ and $Z(t)$ be the random variables which indicate the actions chosen by all nodes and the corresponding rewards, in round $t$. Let $E(k) = \left( A(k), Z(k), \dots, A(1), Z(1) \right)$. Then, $\forall t \geq 1$,
\begin{equation*}
\mathbb{P} [ E(t) = \left( \bar{a}_{1:t}, \bar{z}_{1:t} \right) ] =\tilde{\mathbb{P}} [ E(t) = \left( \bar{a}_{1:t}, \bar{z}_{1:t} \right) ], 
\end{equation*}
where $\bar{a}_{1:t} = \left( \bar{a}_1, \dots, \bar{a}_t \right), \bar{z}_{1:t} = \left( \bar{z}_1, \dots, \bar{z}_t \right)$ with $\bar{a}_k \in \mathcal{K}^m$ and $\bar{z}_k \in [0,1]^m$ for any $k \geq 1$.
\end{lemma} 

\begin{proof}
We establish the result using induction on $t$. The result trivially holds for $t = 1$, since a policy does not possess any information in the very first round itself. Assume that it is true for $t = k$. Then,
\begin{equation*}
\mathbb{P} [ E(k) = \left( \bar{a}_{1:k}, \bar{z}_{1:k} \right) ] =\tilde{\mathbb{P}} [ E(k) = \left( \bar{a}_{1:k}, \bar{z}_{1:k} \right) ]. 
\end{equation*}
Now, we prove that the result holds for $t = k+1$. 
\begin{align}
& \mathbb{P} [ E(k+1)  = \left( \bar{a}_{1:k+1}, \bar{z}_{1:k+1} \right) ] \notag \\
&= \mathbb{P} [ A(k+1) = \bar{a}_{k+1}, Z(k+1) = \bar{z}_{k+1}, E(k) = \left( \bar{a}_{1:k}, \bar{z}_{1:k} \right) ], \notag \\
& =\mathbb{P} [ A(k+1) = \bar{a}_{k+1}, Z(k+1) = \bar{z}_{k+1} \vert E(k) = \left( \bar{a}_{1:k}, \bar{z}_{1:k} \right) ] \notag \\
& \hspace{4cm} \cdot \mathbb{P} [ E(k) = \left( a_{1:k}, \bar{z}_{1:k} \right) ], \notag \\
& =\mathbb{P} [ A(k+1) = \bar{a}_{k+1}, Z(k+1) = \bar{z}_{k+1} \vert E(k) = \left( \bar{a}_{1:k}, \bar{z}_{1:k} \right) ] \notag \\
& \hspace{3.5cm} \cdot \tilde{\mathbb{P}} [ E(k) = \left( \bar{a}_{1:k}, \bar{z}_{1:k} \right) ], \label{eq:apd1}   
\end{align} 
since we assumed that the result is true for $t = k$. Note that, in our model, the actions taken by a policy in round $(k+1)$ for a given $E(k)$, are independent of the probability space from which the rewards are generated. Further, the reward distributions of arms are identical in both probability spaces $\mathbb{P}$ and $\tilde{\mathbb{P}}$. Therefore, 
\begin{align}
\label{eq:apd2}
& \mathbb{P} [ A(k+1) = \bar{a}_{k+1}, Z(k+1) = \bar{z}_{k+1} \vert E(k) = \left( \bar{a}_{1:k}, \bar{z}_{1:k} \right) ] \notag \\
& = \tilde{\mathbb{P}} [ A(k+1) = \bar{a}_{k+1}, Z(k+1) = \bar{z}_{k+1} \vert E(k) = \left( \bar{a}_{1:k}, \bar{z}_{1:k} \right) ]. 
\end{align}
By substituting \eqref{eq:apd2} in \eqref{eq:apd1}, we obtain 
\begin{multline*}
\mathbb{P} \left[ E(k+1) = \left( \bar{a}_{1:k+1}, \bar{z}_{1:k+1} \right)  \right] = \\ \tilde{\mathbb{P}} \left[ E(k+1) = \left( \bar{a}_{1:k+1}, \bar{z}_{1:k+1} \right) \right],
\end{multline*}
which completes the proof. 
\end{proof}
\begin{lemma}
\label{Thm:A.1}
Let $c_{t,S} = \sqrt{\frac{2 \ln t}{S}}, \beta \in (0,1)$. For each $v \in V$ and sub-optimal arm $i$,  define $\tau^v_i$ as follows: 
\begin{equation*}
\tau^v_i = \min \lbrace t \in [n]: m^v_i(t) \geq l_i = \Big \lceil \frac{8\ln n}{\Delta_i^2} \Big \rceil \rbrace.
\end{equation*} 
Then, for each $t > \tau^v_i$,
\begin{multline*}
\mathbb{P} \left( \lbrace \hat{\mu}_{m^v_*(t)} + c_{t,m^v_*(t)} \leq \hat{\mu}_{m^v_i(t)}  + c_{t,m^v_i(t)} \rbrace \right) \\
\leq 2 \left( \frac{\ln t}{\ln \left(1/\beta \right)} + 1 \right) \frac{1}{t^{4 \beta}}. 
\end{multline*}
\end{lemma}
\begin{proof}
For convenience,  we denote  
\begin{equation*}
A^v_i(t) = \lbrace \hat{\mu}_{m^v_*(t)} + c_{t,m^v_*(t)} \leq \hat{\mu}_{m^v_i(t)} + c_{t,m^v_i(t)} \rbrace.
\end{equation*}
Note that,
\begin{equation}
\mathbb{P} \left( A^v_i(t) \cap \lbrace t > \tau^v_i \rbrace \right) = \mathbb{P} \left( A^v_i(t) \cap \lbrace m^v_i(t) \geq l_i \rbrace \right).
\end{equation}
Observe that, the event $A^v_i(t)$ occurs only if atleast one of the following events occur.
\begin{align}
\lbrace \hat{\mu}_{m^v_*(t)}  &\leq \mu^* - c_{t, m^v_*(t)} \rbrace, \label{eq:4.10} \\
\lbrace \hat{\mu}_{m^v_i(t)}  &\geq \mu_i + c_{t, m^v_i(t)} \rbrace, \label{eq:4.11} \\
\lbrace \mu^* &< \mu_i + 2c_{t,m^v_i(t)} \rbrace. \label{eq:4.12}
\end{align}
Note that, the event given by \eqref{eq:4.12} does not occur when the event $\lbrace m^v_i(t) \geq l_i \rbrace$ occurs. Hence, 
\begin{align}
& \mathbb{P} \left( A^v_i(t) \cap \lbrace m^v_i(t) \geq l_i \rbrace \right) \leq \notag \\
&\mathbb{P} ( \lbrace \hat{\mu}_{m^v_*(t)}  \leq \mu^* - c_{t, m^v_*(t)} \rbrace \cup \lbrace \hat{\mu}_{m^v_i(t)}  \geq \mu_i + c_{t, m^v_i(t)} \rbrace \notag \\
&\hspace{4cm}\cap \lbrace m^v_i(t) \geq l_i \rbrace ), \notag \\
& \leq\mathbb{P} \left( \lbrace \hat{\mu}_{m^v_*(t)}  \leq \mu^* - c_{t, m^v_*(t)} \rbrace  \right) \notag \\
& \hspace{2.5cm}+ \mathbb{P} \left( \lbrace \hat{\mu}_{m^v_i(t)} \geq \mu_i + c_{t, m^v_i(t)} \rbrace \right).
\label{eq:qq}
\end{align}
For each node $v \in V$ and each arm $i$, the initialization phase of the UCB-user policy implies that $\vert \mathcal{N}(v) \vert \leq m^v_i(t) \leq \vert \mathcal{N}(v) \vert t $. Therefore, 
\begin{align*}
& \mathbb{P}  \left(  \hat{\mu}_{m^v_*(t)}  \leq \mu^* - c_{t, m^v_*(t)}  \right) \leq \notag \\
& \mathbb{P} \left( \exists s_* \in \lbrace \vert \mathcal{N}(v) \vert, \dots, \vert \mathcal{N}(v) \vert t \rbrace : \hat{\mu}_{s_*}  \leq \mu^* - c_{t, s_*} \right),
\end{align*}
\begin{align}
& \leq \sum\limits_{j = 0}^{\frac{\ln t}{\ln \left( 1/\beta \right)}} \mathbb{P} \Big( \exists s_* : \vert \mathcal{N}(v) \vert \beta^{j+1} t < s_* \leq \vert \mathcal{N}(v) \vert \beta^j t, \notag\\
& \hspace{1.5cm} s_*\hat{\mu}_{s_*}  \leq s_*\mu^* - \sqrt{2s_* \ln t} \Big), \label{eq:peel1} \\
& \leq \sum\limits_{j = 0}^{\frac{\ln t}{\ln \left(1/\beta \right)}} \mathbb{P} \Big( \exists s_* : \vert \mathcal{N}(v) \vert \beta^{j+1} t < s_* \leq \vert \mathcal{N}(v) \vert \beta^j t, \notag\\
& \hspace{1.5cm} s_* \hat{\mu}_{s_*} \leq s_* \mu_* - \sqrt{2 \vert \mathcal{N}(v) \vert \beta^{j+1} t \ln t} \Big). \label{eq:peel2}
\end{align}
Here, (\ref{eq:peel1}) is due to the peeling argument on geometric grid over $[\vert \mathcal{N}(v)\vert, \vert \mathcal{N}(v) \vert t]$. This implies that, for $\beta \in(0,1)$, $a \geq 1$, if $s \in \lbrace a, \dots, at \rbrace$ then there exists $j \in \lbrace 0, \dots, \frac{\ln t}{\ln \left(1/\beta \right)} \rbrace$ such that $a \beta^{j+1} t < s \leq a \beta^j t$. Now, we proceed to bound the probability of the event given by (\ref{eq:peel2}) using Hoeffding's maximal inequality and Lemma \ref{Lemma:coupling}. Hence,
\begin{align}
\mathbb{P} \left( \hat{\mu}_{m^v_*(t)} \leq \mu^* - c_{t, m^v_*(t)}  \right) & \leq \sum\limits_{j = 0}^{\frac{\ln t}{\ln \left(1/\beta \right)}} \exp \left( -4 \beta \ln t \right), \notag \\
& \leq \left( \frac{\ln t}{\ln \left(1/\beta \right)} + 1 \right) \frac{1}{t^{4\beta}}. \label{eq:CH1}
\end{align}
Similarly, we can show that 
\begin{equation}
\label{eq:CH2}
\mathbb{P} \left( \hat{\mu}_{m^v_i(t)}  \geq \mu_i + c_{t, m^v_i(t)}  \right) \leq \left( \frac{\ln t}{\ln \left(1/\beta \right)} + 1 \right) \frac{1}{t^{4\beta}}.
\end{equation}
Substituting \eqref{eq:CH1} and \eqref{eq:CH2} in (\ref{eq:qq}) gives the desired result.
\end{proof}
\begin{lemma}
\label{Thm:A.2}
Let $\tau^v_i$ $\forall v \in V,$ and $l_i$,  $\forall 1 \leq i \leq K$ be as defined in the Lemma~\ref{Thm:A.1}. Assume that a node $v$ stops playing the sub-optimal arm $i$ at time $\tau^v_i$. Then, for an arm $i$, $\sum\limits_{v\in V} T^v_i(\tau^v_i) \leq C_G  l_i$, where $C_G l_i$ is the solution to the optimisation problem in~\eqref{OptimizationProblem11}.
\end{lemma}
\begin{proof}
We first evaluate the value of the random variable $\sum\limits_{v=1}^m T^v_i(\tau^v_i)$ for all realizations. Then, we determine the maximum value of the random variable over all realizations. The following algorithm gives the value of the above mentioned random variable for a realization. Consider an $m$ length column vector of zeros, say $y$. \\
\emph{Algorithm}:\\
Step 1: Select an integer $I$ from $B = \lbrace 1, 2, \dots, m \rbrace$. \\
Step 2: Increase $y(I)$ by 1, \textit{i.e.,} $y(I) = y(I) + 1$. \\
Step 3: Find the indices (say $C$) corresponding to elements in $Ay$ which are atleast $l_i$. Here, $A$ is the adjacency matrix of the graph $G$. \\ 
Step 4: Update $B = B \setminus C$ and $A$ by removing rows corresponding to $C$ in $A$ \\
Step 5: Go to step 1, if $B$ is non-empty else stop by returning~$y$.\\
Here, step 4 ensures that nodes having $l_i$ samples of arm $i$ stops playing arm $i$ further. Observe that $\Vert y \Vert_1$, where $y$ is the vector returned by the above algorithm, yields the value of the random variable $\sum\limits_{v=1}^m T^v_i(\tau^v_i)$ for a realization. Therefore, it suffices to maximize $\Vert y \Vert_1$ over all realizations.\\
The optimisation problem in $\eqref{OptimizationProblem11}$ captures the above. The final constraint in $\eqref{OptimizationProblem11}$ ensures that the node $\eta_k$ has $l_i$ samples of sub-optimal arm $i$ at time instance $\gamma_k$. Recall that, $\gamma_k$ is a random variable which tracks the \textit{least} time at which atleast $k$ nodes have more than $l_i$ samples of arm $i$. The penultimate constraint ensures that sub-optimal arm $i$ count of node $\eta_k$ does not increase(or stop playing arm $i$) after time instance $\gamma_k$. Hence, a feasible point in the above optimisation problem is a sequence $\lbrace z_k \rbrace_{k=1}^m$ which satisfies the aforementioned two constraints. Then, $\Vert z_m \Vert_1$   corresponds to the value of the random variable $\sum\limits_{v=1}^m T^v_i(\tau^v_i)$ for a realization. 
\end{proof}
By using the above lemmas, we now prove Theorem~\ref{Thm:3.1}. 
\begin{proof}
From \eqref{eq:2.3}, we need to upper bound $\mathbb{E}[ T_i^v(n) ]$ for all $v\in V$ in order to upper bound the expected regret of $G$. Let $B^v_i(t)$ be the event that node-$v$ plays sub-optimal action-$i$ in round $t$:
\begin{align}
B^v_i(t) &= \lbrace \hat{\mu}_{m^v_j(t)} + c_{t,m^v_j(t)} \leq \hat{\mu}_{m^v_i(t)} + c_{t,m^v_i(t)}, \forall j \neq i \rbrace, \notag\\
& \subseteq \lbrace \hat{\mu}_{m^v_*(t)} + c_{t,m^v_*(t)} \leq \hat{\mu}_{m^v_i(t)} + c_{t,m^v_i(t)} \rbrace.
\end{align} 
Hence,
\begin{align}
\mathbb{E} \left[ \sum_{v=1}^m T_i^v(n) \right] &= \mathbb{E} \left[ \sum_{v=1}^m \sum_{t=1}^n [ \mathbb{I}_{\lbrace t \leq \tau^v_i, B^v_i(t) \rbrace }  + \mathbb{I}_{\lbrace t > \tau^v_i, B^v_i(t) \rbrace } ] \right], \notag  \\ 
& \hspace{-2cm}= \underbrace{\mathbb{E} \left[ \sum_{v=1}^m T_i^v(\tau^v_i) \right]}_{(a)} + \underbrace{\mathbb{E} \left[ \sum_{v=1}^m \sum_{t=1}^n \mathbb{I}_{\lbrace t > \tau^v_i, B^v_i(t) \rbrace}  \right]}_{(b)}. \label{NumberOfArmPlays2}
\end{align}
Now, we upper bound $(b)$ in~\eqref{NumberOfArmPlays2}. Let $ 1 \leq v \leq m$. Since, $m^v_i(t) \geq l_i$ for $t > \tau^v_i$,
\begin{align*}
&\mathbb{E} \left[ \sum_{t=1}^n \mathbb{I}_{\lbrace t > \tau^v_i \rbrace } \mathbb{I}_{B^v_i(t)} \right] = \sum_{t=1}^n \mathbb{P} \left( B_i^v(t), \{t > \tau^v_i \} \right),   \\
& \overset{(c)}{\leq} \sum\limits_{t=1}^\infty 2 \left( \frac{\ln t}{\ln \left(1/\beta \right)} + 1 \right) \frac{1}{t^{4\beta}}, \\
& \leq \int\limits_{1}^\infty 2 \left( \frac{\ln t}{\ln \left(1/\beta \right)} + 1 \right) \frac{1}{t^{4\beta}} \,\, \mathrm{d}t, \\
& = \frac{2}{4\beta -1} + \frac{2}{(4\beta-1)^2 \ln(1/\beta)},
\end{align*}
where $(c)$ is due to Lemma~\ref{Thm:A.1}. Thus, $(b)$ in \eqref{NumberOfArmPlays2} upper bounded as
\begin{multline}
\label{eq:4.17}
\mathbb{E} \left[ \sum\limits_{v=1}^m \sum_{t=1}^n \mathbb{I}_{\lbrace t > \tau^v_i \rbrace} \mathbb{I}_{B^v_i(t)} \right] \\
\leq m \left( \frac{2}{4\beta -1} + \frac{2}{(4\beta-1)^2 \ln(1/\beta)} \right).
\end{multline}
Now, we upper bound the random variable in $(a)$ in \eqref{NumberOfArmPlays2} for all realizations.  Consider a new system in which each node $v$ stops playing sub-optimal arm $i$ for $t>\tau^v_i$. By using Lemma~\ref{Thm:A.2}, we can calculate an upper bound on~$ \sum\limits_{v=1}^m T^v_i(\tau^v_i)$. It is easy to see that the same upper bound also holds for $(a)$ in \eqref{NumberOfArmPlays2}. Hence, 
\begin{equation}
\label{eq:4.18}
\mathbb{E} \left[ \sum_{v=1}^m T_i^v(\tau^v_i) \right] \leq C_G l_i.
\end{equation}
Combining \eqref{NumberOfArmPlays2}, \eqref{eq:4.17} and \eqref{eq:4.18} establishes the desired result.
\end{proof}
\begin{lemma}
\label{Thm:NAIC}
Consider a network $G = (V,E)$ learning a $K$-arm stochastic MAB problem with mean rewards $\mu_1~\geq~\mu_2~\geq~\dots~\mu_K$. Assume that, each arm distribution is discrete and it assigns a non-zero probability to each possible value. 
Then, the UCB-user policy followed by any user $v$ in $G$ to learn the above MAB problem is non-altruistic and individually consistent (NAIC) policy.
\end{lemma}
\begin{proof}
First, we prove the non-altruistic part. Lemma~\ref{Thm:A.1} gives an upper bound on the probability that a node $v$ following the UCB-user policy plays any sub-optimal arm $i$ in round $t$, after it has obtained $l_i = \frac{8 \ln n}{\Delta_i^2}$ samples of the arm $i,$ where $n$ is the time horizon. We can treat $\frac{8}{\Delta_i^2}$ as $a$ in the definition of non-altruistic policy. Observe that, in \eqref{eq:4.17}, we upper bounded the expected number of times a node $v$ chooses any sub-optimal arm $i$, after it has access to $l_i$ samples of arm $i$, till $n$. Note that, this upper bound is a constant. Hence, the UCB-user policy satisfies the non-altruistic property. Now, we prove the \network-oblivious part. Recall that, $\omega_{\bar{v}}$ contains actions and the corresponding rewards of the nodes outside the neighbourhood of node $v$, from round $1$ to $n$. Note that,  the event $A^v_i(t)$ defined in the proof of Lemma~\ref{Thm:A.1} is independent of any $\omega_{\bar{v}}$, given the event $\lbrace m^v_i(t) = a, m^v_*(t) =b \rbrace$. Hence, on the lines of Lemma~\ref{Thm:A.1}, for $\beta \in (0,1)$, $t > \tau^v_i$ (same as defined in Lemma~\ref{Thm:A.1}),
\begin{equation}
\mathbb{P} \left( A^v_i(t) | \omega_{\bar{v}} \right) \leq 2 \left(  \frac{\ln t}{\ln (1/\beta)} + 1 \right) \frac{1}{t^{4\beta}}.
\end{equation} 
Thus,
\begin{align*}
\mathbb{E}[T^v_i(n) | \omega_{\bar{v}}] & \leq l_i + \sum \limits_{t = \tau^v_i+1} \mathbb{P} \left( A^v_i(t) | \omega_{\bar{v}} \right), \\
& \leq \frac{8 \ln n}{\Delta_i^2} + \sum\limits_{t=1}^\infty \mathbb{P} \left( A^v_i(t) | \omega_{\bar{v}} \right), \\
& = \left( \frac{8 \ln n}{\Delta_i^2} + O(1) \right) \in o(n^c), \,\, \text{for any}\,\, c > 0.  
\end{align*}
Therefore, the UCB-user policy followed by a node $v$ satisfy \network-oblivious property, which completes the proof.  
\end{proof}
\section*{Appendix B}
\label{Universal Lower Bound}
\emph{Proof of Theorem~\ref{Thm:4.4}.}\\
\begin{proof}
Follows from Theorem~2 in~\cite{Lai}, by considering $m_i^G(n)$ instead of $T_i(n)$ in the event $C_n$ defined in the respective proof.
\end{proof}
\section*{Appendix C}
\label{NAIC Lower Bound}
\emph{Proof of Theorem~\ref{Thm:4.1}.}\\
We now prove $(i)$ in Theorem~\ref{Thm:4.1}, in the following lemma. With the aid of this lemma, we then prove the second part of the theorem.
\begin{lemma}
\label{Thm:4.2}
Consider a node $v$ in a network $G$. Assume that node $v$ follows an NAIC policy, and suppose [A1] holds. Further, assume that each arm is associated with a discrete distribution such that it assigns a non-zero positive probability to each possible value. Then, for any $\boldsymbol{\theta} \in \boldsymbol{\Theta_j}$, and for any $\omega_{\bar{v}}$,  the following holds:
\begin{align*}
\liminf \limits_{n \rightarrow \infty} \frac{\mathbb{E}_{\boldsymbol{\theta}} [m_j^v(n) \vert \omega_{\bar{v}} ]}{\ln n} & \geq \frac{1-\delta}{1 + \delta} \cdot \frac{1}{kl(\theta_j || \theta_1)}, \\
\liminf \limits_{n \rightarrow \infty} \frac{\mathbb{E}_{\boldsymbol{\theta}} [m_j^v(n) ]}{\ln n} & \geq \frac{1-\delta}{1 + \delta} \cdot \frac{1}{kl(\theta_j || \theta_1)}.
\end{align*}
\end{lemma} 
\begin{proof}
Without loss of generality, assume that $\theta_1 = \theta^*$ and $j = 2 \Rightarrow \boldsymbol{\theta} \in \boldsymbol{\Theta_2}$. Consider a new parameter vector $\boldsymbol{\gamma} = \left( \theta_1, \lambda, \theta_3, \dots, \theta_K \right)$ such that $\mu(\lambda) > \mu(\theta^*)$, $j \neq 1$. Note that, arm 1 is optimal under parameter vector $\boldsymbol{\theta}$, while arm 2 is optimal under parameter vector $\boldsymbol{\gamma}$. Let $X_{2,1}, \dots, X_{2,n}$ be $n$ i.i.d samples generated from the sub-optimal arm 2's distribution with parameter vector $\boldsymbol{\theta}$. Define 
\begin{equation*}
\hat{kl}_s = \sum\limits_{t=1}^s \ln \left( \frac{ f(X_{2,t};\theta_2) }{ f(X_{2,t};\lambda)} \right).
\end{equation*}
For any $v \in V$ and any sub-optimal arm $j$, and $0 < a < \delta$, we define 
\begin{equation}
C_n^v = \lbrace m^v_2(n) < \frac{(1-\delta)\ln n}{kl(\theta_2 || \lambda)} \,\, \text{and} \,\, \hat{kl}_{m^v_2(n)} \leq (1-a) \ln n \rbrace,
\end{equation}
where $ \hat{kl}_{m_2^v(n)} = \sum\limits_{u \in \mathcal{N}(v)} \sum\limits_{t=1}^{T_2^u(n)} \ln \left( \frac{ f(X_{2,t}^u;\theta_2) }{ f(X_{2,t}^u;\lambda)} \right)$, since $\lbrace X^u_{2,t} \rbrace_{u \in \mathcal{N}(v)}$ are i.i.d. For convenience, let $g_n = \frac{(1-\delta) \ln n}{kl(\theta_2 || \lambda)}$ and $h_n = (1-a) \ln n$. For a given $\omega_{\bar{v}}$, observe that $C_n^v$ is a disjoint union of events of the form $\lbrace m^v_1(n) = n_1, m^v_2(n) = n_2, \dots, m^v_K(n) = n_K, \hat{kl}_{n_2} \leq h_n \rbrace$ with $n_1 + n_2 \dots + n_K = n|\mathcal{N}(V)|$ and $n_2 \leq g_n$. Further, $\lbrace m_2^v(n) = n_2 \rbrace$  is also a disjoint union of the events of the form  $\lbrace \cap_{u \in \mathcal{N}(v)} T^u_2(n)=q_u \rbrace$  with $ \sum\limits_{u \in \mathcal{N}(v)} q_u = n_2$. Since $\boldsymbol{\gamma} = (\theta_1, \lambda, \theta_3, \dots,\theta_K)$ and $\boldsymbol{\theta} = (\theta_1, \theta_2, \theta_3, \dots,\theta_K)$, we write  
\begin{multline}
\mathbb{P}_{\boldsymbol{\gamma}} \lbrace m^v_1(n) = n_1, \dots, m^v_K(n) = n_K, \hat{kl}_{n_2} \leq h_n \vert \omega_{\bar{v}} \rbrace  = \\
\mathbb{E}_{\theta} \Bigg[ \mathbb{I}_{ \lbrace m^v_1(n) = n_1, \dots , m^v_K(n) = n_K, \hat{kl}_{n_2} \leq h_n \rbrace} \\ \prod\limits_{u \in \mathcal{N}(v)} \prod\limits_{t=1}^{T_2^u(n) = q_u} \frac{ f(X_{2,t}^u;\lambda) }{f(X_{2,t}^u;\theta_2)} \Bigg].
\end{multline}
However, $\prod\limits_{u \in \mathcal{N}(v)} \prod\limits_{t=1}^{q_u} \frac{ f(X_{2,t}^u;\lambda) }{f(X_{2,t}^u;\theta_2)} = \exp(-\hat{kl}_{n_2})$. Therefore,
\begin{multline*}
\mathbb{P}_{\boldsymbol{\gamma}} \lbrace m^v_1(n) = n_1, \dots, m^v_K(n) = n_K, \hat{kl}_{n_2} \leq h_n \vert \omega_{\bar{v}} \rbrace  = \\
\mathbb{E}_{\theta} \left[ \mathbb{I}_{ \lbrace m^v_1(n) = n_1, \dots, m^v_K(n) = n_K, \hat{kl}_{n_2} \leq h_n \rbrace}    \exp(-\hat{kl}_{n_2}) \right] .
\end{multline*}
Note that, $\exp(-\hat{kl}_{n_2}) \geq n^{-(1 - a)}$, since $\hat{kl}_{n_2} \leq h_n$ in the region of integration. Therefore,
\begin{multline}
\mathbb{P}_{\boldsymbol{\gamma}} \lbrace m^v_1(n) = n_1, \dots,  m^v_K(n) = n_K,  \\ \hat{kl}_{n_2} \leq h_n \vert \omega_{\bar{v}} \rbrace  \\
\geq n^{-(1 - a)} \mathbb{P}_{\boldsymbol{\theta}} \lbrace m^v_1(n) = n_1, \dots, m^v_K(n) = n_K, \\ \hat{kl}_{n_2} \leq h_n \vert \omega_{\bar{v}} \rbrace.
\end{multline}
Hence,
\begin{equation}
\mathbb{P}_{\boldsymbol{\gamma}} (C^v_n \vert \omega_{\bar{v}}) \geq n^{-(1 - a)} \mathbb{P}_{\boldsymbol{\theta}} (C^v_n \vert \omega_{\bar{v}} ).
\end{equation}
Now, we bound $\mathbb{P}_{\boldsymbol{\theta}} (C^v_n \vert \omega_{\bar{v}} )$ as follows:
\begin{equation*}
\mathbb{P}_{\boldsymbol{\gamma}} (C^v_n \vert \omega_{\bar{v}} ) \leq \mathbb{P}_{\boldsymbol{\gamma}} \left( m^v_2(n) < g_n \vert \omega_{\bar{v}} \right). 
\end{equation*}
Since $T_2^v(n) \leq m_2^v(n)$,
\begin{align*}
\mathbb{P}_{\boldsymbol{\gamma}} (C^v_n \vert \omega_{\bar{v}} ) & \leq \mathbb{P}_{\boldsymbol{\gamma}} \left( T^v_2(n) < g_n \vert \omega_{\bar{v}} \right), \\
& = \mathbb{P}_{\boldsymbol{\gamma}} \left( n - T^v_2(n) > n - g_n \vert \omega_{\bar{v}} \right). 
\end{align*}
Note that, $n \vert \mathcal{N}(v) \vert - m^v_2(n)$ is a non-negative random variable and $kl(\theta_2 || \lambda) > 0$. Therefore, applying Markov's inequality to the right-hand side in the above equation, we obtain
\begin{align*}
\mathbb{P}_{\boldsymbol{\gamma}} (C^v_n \vert \omega_{\bar{v}} ) & \leq \frac{\mathbb{E}_{\boldsymbol{\gamma}}[n - T^v_2(n) \vert \omega_{\bar{v}} ] }{n - g_n}, \\
& = \frac{\sum\limits_{i=1, i \neq 2}^K \mathbb{E}_{\boldsymbol{\gamma}}[ T^v_i(n) \vert \omega_{\bar{v}} ]  }{n - g_n} = \frac{(K-1) o(n^a)}{n - O(\ln n)},
\end{align*}
for $0 < a < \delta$, since arm 2 is the unique optimal arm under $\gamma$. Hence,
\begin{equation} 
\label{a1}
\mathbb{P}_{\boldsymbol{\theta}} (C^v_n \vert \omega_{\bar{v}} )  \leq n^{(1-a)} \mathbb{P}_{\boldsymbol{\gamma}} (C^v_n \vert \omega_{\bar{v}} ) = o(1).   
\end{equation}
Observe that,
\begin{multline}
\mathbb{P}_{\boldsymbol{\theta}} \left( C^v_n \vert \omega_{\bar{v}} \right) \geq \mathbb{P}_{\boldsymbol{\theta}} \Big( m_2^v(n) < g_n, \\
\frac{1}{g_n}  \max\limits_{i \leq  g_n } \hat{kl}_i \leq \frac{kl(\theta_2||\lambda)(1-a)}{(1-\delta)} \Big\vert \omega_{\bar{v}} \Big), \label{a2}
\end{multline}
\begin{equation}
\label{a3}
\mathbb{P}_{\boldsymbol{\theta}} \left( \frac{1}{g_n} \max\limits_{i \leq  g_n } \hat{kl}_i \leq \frac{kl(\theta_2||\lambda)(1-a)}{(1-\delta)}  \right) \rightarrow 1,
\end{equation}
due to $\frac{1-a}{1-\delta} > 1$ and the maximal version of the Strong Law of Large Numbers which is given below. 

\emph{Maximal version of SLLN}~\cite{Bubeck}: Let $\lbrace X_t \rbrace$ be a sequence of independent real-valued random variables with positive mean $\mu > 0$. Then, 
\begin{equation*}
\lim\limits_{n \rightarrow \infty} \frac{1}{n} \sum\limits_{t=1}^n X_t = \mu \,\, a.s. \Rightarrow \lim\limits_{n \rightarrow \infty} \frac{1}{n} \max_{s= 1, \dots, n} \sum\limits	_{t=1}^s X_t = \mu\,\, a.s.
\end{equation*}
From \eqref{a1}, \eqref{a2} and \eqref{a3}, we obtain 
\begin{align*}
\mathbb{P}_{\boldsymbol{\theta}} \left( m^v_2(n) < g_n \vert \omega_{\bar{v}} \right) & = o(1), \,\, \forall \omega_{\bar{v}}, \\
 \Rightarrow \mathbb{P}_{\boldsymbol{\theta}} \left( m^v_2(n) < g_n \right) & = o(1).
\end{align*}
Part (iii) of assumption, $[A1]$, guarantees the existence of a $\lambda \in \Theta$ such that $\mu(\theta_1) < \mu(\lambda) < \mu(\theta_1) + \delta$ holds. Combining $\mu(\theta_1) > \mu(\theta_2)$ with the part (i) of $[A1]$, we obtain $0 < kl(\theta_2||\theta_1) < \infty$. From part (ii) of $[A1]$, we deduce that $| kl(\theta_2||\theta_1) - kl(\theta_2||\lambda)| < \epsilon $, since $\mu(\theta_1) \leq \mu(\lambda) \leq \mu(\theta_1) + \delta$ for some $\delta$. Let $\epsilon$ be $\delta kl(\theta_2 || \theta_1)$. Hence, we write the following:
\begin{equation*}
|kl(\theta_2 || \lambda) - kl(\theta_2 || \theta_1) | < \delta kl(\theta_2||\theta_1), \hspace{0.5cm} \text{for} \,\, 0 < \delta < 1.
\end{equation*}  
Hence,
\begin{align*}
\mathbb{P}_{\boldsymbol{\theta}} \left( m^v_2(n) < \frac{1-\delta}{1 + \delta} \cdot \frac{\ln n}{kl(\theta_2 || \theta_1)} \Big\vert \omega_{\bar{v}} \right) &= o(1),  \\
\Rightarrow \mathbb{P}_{\boldsymbol{\theta}} \left( m^v_2(n) < \frac{1-\delta}{1 + \delta} \cdot \frac{\ln n}{kl(\theta_2 || \theta_1)} \right) &= o(1). 
\end{align*}
Furthermore,
\begin{align*}
& \mathbb{E}_{\boldsymbol{\theta}} [m_2^v(n) \vert \omega_{\bar{v}} ] = \sum\limits_i i \cdot \mathbb{P}_{\boldsymbol{\theta}} \left( m_2^v(n) = i \vert \omega_{\bar{v}} \right),  \\
& \geq \left( \frac{1-\delta}{1 + \delta} \right) \frac{\ln n}{kl(\theta_2 || \theta_1)} \mathbb{P}_{\boldsymbol{\theta}} \left( m_2^v(n) > \frac{1-\delta}{1 + \delta} \cdot \frac{\ln n}{kl(\theta_2 || \theta_1)}  \Big\vert \omega_{\bar{v}} \right), \\ 
&= \left( \frac{1-\delta}{1 + \delta} \right) \frac{\ln n}{kl(\theta_2 || \theta_1)}  (1-o(1)).
\end{align*}
Hence, we have proved that for any $v \in V$, $\omega_{\bar{v}}$ and any sub-optimal arm $j$, 
\begin{align*}
\liminf \limits_{n \rightarrow \infty} \frac{\mathbb{E}_{\boldsymbol{\theta}} [m_j^v(n) \vert \omega_{\bar{v}} ]}{\ln n} & \geq \frac{1-\delta}{1 + \delta} \cdot \frac{1}{kl(\theta_j || \theta_1)}, \\
\liminf \limits_{n \rightarrow \infty} \frac{\mathbb{E}_{\boldsymbol{\theta}} [m_j^v(n) ]}{\ln n} & \geq \frac{1-\delta}{1 + \delta} \cdot \frac{1}{kl(\theta_j || \theta_1)},
\end{align*}
which completes the proof of this lemma, and establishes $(i)$ in Theorem~\ref{Thm:4.1}.
\end{proof}
With the help of this, we now prove the second part of Theorem ~\ref{Thm:4.1}.

\begin{proof}
Lemma~\ref{Thm:4.2} implies that for each $v \in V$, $\exists \, n_v \in \mathbb{N}$ such that 
\begin{equation}
\label{eq:asd}
\frac{\mathbb{E}_{\boldsymbol{\theta}} [m_j^v(n) ]}{\ln n} \geq \frac{1-\delta}{1 + \delta} \cdot \frac{1}{kl(\theta_j || \theta_1)}, \,\,\, \forall n \geq n_v.
\end{equation} 
Let $n' = \max(n_v : v \in V)$. Using (\ref{eq:asd}) for each $v \in V$, and for $n \geq n'$, we determine a lower bound for $\mathbb{E}_{\boldsymbol{\theta}} [m_j^G(n) ]$. It is easy to see that the solution to the following optimisation problem is a valid lower bound for $\mathbb{E}_{\boldsymbol{\theta}} [m_j^G(n) ]$ for $n \geq n'$.  
\begin{equation}
\label{OptimizationProblem3}
\begin{aligned}
& \underset{}{\text{minimize}} \hspace{2mm} \Vert z_m \Vert_1 \\
& \text{s.t $\exists$ a sequence $\lbrace z_k \rbrace_{k=1}^m$} \\
& z_i(\eta_k) = z_k(\eta_k) \hspace{2mm} \forall i \geq k ,\\
& \langle z_k, A(n_k,:) \rangle \geq \frac{1-\delta}{1 + \delta} \cdot \frac{1}{kl(\theta_j, \theta_1)} \ln n \hspace{2mm} \forall k.
\end{aligned}
\end{equation}
Note that, the notation in (\ref{OptimizationProblem3}) is same as used in Theorem~\ref{Thm:3.1}, Lemma~\ref{Thm:A.2}. Let $L_G \left( \frac{1-\delta}{1+\delta}\right) \frac{\log n}{kl(\theta_j || \theta_1)}$  be the solution of (\ref{OptimizationProblem3}). Thus,
\begin{eqnarray*}
\mathbb{E}[m^G_j(n)] \geq L_G \left( \frac{1-\delta}{1+\delta} \right) \frac{\ln n}{kl(\theta_j || \theta_1)}, \,\,\, \forall n \geq n', \\
\Rightarrow \liminf\limits_{n \rightarrow \infty} \frac{\mathbb{E}[m^G_j(n)]}{\ln n} \geq L_G \left( \frac{1-\delta}{1+\delta}\right) \frac{1}{kl(\theta_j || \theta_1)},
\end{eqnarray*}
which establishes the desired result.
\end{proof}
\section*{Appendix D}
\label{NAIC Lower Bound Special Case}
\emph{Proof of Theorem~\ref{Thm:4.3}}.\\
\begin{proof}
Without loss of generality we consider that node 1 is the center node and node 2 through $m_n$ are leaf nodes. Since a policy does not possess any information in the first round, it chooses arm 1 with probability $p_1$ and arm 2 with probability $p_2$, such that $0 \leq p_1, p_2 \leq 1$ and $p_1 + p_2 = 1$. Now, we find the expected number of nodes that chose the arm with parameter $\mu_b$ in the first round  as follows:
\begin{equation}
\mathbb{E}[m^{G_n}_b(1)] = \sum\limits_{v \in V} \left( \frac{1}{2} p_2 + \frac{1}{2} p_1 \right) = \frac{m_n}{2} \geq \frac{\ln n}{kl(\mu_b, \mu_a)},
\end{equation} 
since MAB is $(\mu_a, \mu_b)$ with probability $\frac{1}{2}$, and is $(\mu_b, \mu_a)$ with probability $\frac{1}{2}$. Henceforth, for convenience, we replace $a$ with 1 and $b$ with 2. Let $m^{G,v}_i(t)$ be a random variable indicating the total number of times arm $i$ has been chosen by node $v$ and its one hop neighbours till round $t$, in the network $G$. Note that, $m^{G_n}_2(1)$ is equals to $m^{G_n,1}_2(1)$, since the network in consideration is a star network with node 1 as the center node. Therefore, 
\begin{equation}
\label{eq:LargeStar(a)}
\mathbb{E} [m_2^{G_n,1}(1)] \geq \frac{\ln n}{kl(\mu_2, \mu_1)},  \,\,\,\, \forall n \in \mathbb{N},
\end{equation}
From Theorem~\ref{Thm:4.1}, it follows that
\begin{equation}
\label{eq:LargeStar(b)}
\liminf\limits_{n \rightarrow \infty} \frac{\mathbb{E} [m_2^{G_n,v}(n)]}{\ln n} \geq  \frac{1}{kl(\mu_2, \mu_1)}, \,\,\, \forall \,\, v \in V_n.
\end{equation}
The above inequalities imply that, for any $v\in V_n$, $\exists\, n_v \in \mathbb{N}$ such that $ \frac{\mathbb{E} [m_2^{G_n,v}(n)]}{\ln n} \geq \frac{1}{kl(\mu_2, \mu_1)} \forall n \geq n_v$. Let $n' = \max(n_v:v\in V_n)$. \\
For all $n \in \mathbb{N}$, since the center node has obtained $\frac{\ln n}{kl(\mu_2, \mu_1)}$ samples of arm 2 in the very first round, and the policy is non-altruistic, it chooses arm 2 at most $O(1)$ number of times further. For all $n \geq n'$, in order to satisfy all the inequalities in (\ref{eq:LargeStar(b)}), each leaf node has to choose the arm 2 at least $\left( \frac{\ln n'}{kl(\mu_2, \mu_1)} - O(1) \right)$ times. Hence,
\begin{align*}
\mathbb{E} [m_2^{G_{n}}(n)]  \geq \left( m_n - 1 \right) & \left( \frac{\ln n}{kl(\mu_2, \mu_1)} - O(1) -1 \right) \, \forall n \geq n',\\
\Rightarrow \liminf\limits_{n \rightarrow \infty} \frac{\mathbb{E} [m_2^{G_n}(n)]}{(m_n - 1) \ln n}  & \geq \frac{1}{kl(\mu_2, \mu_1)},
\end{align*}
which completes the proof. 
\end{proof}
\section*{Appendix E}
\emph{Proof of Theorem~\ref{Thm:5.1}}.\\
\begin{proof}
Without loss of generality, we assume that node 1 is the center node in the star network. Under FYL policy, for $2 \leq u \leq m$, $a^u(t) = a^1(t-1)$ for $t >1$. Hence, for any sub-optimal arm $i$,
\begin{align*}
T_i^u(n) = & \mathbb{I}_{\lbrace a^u(1) = i \rbrace} +  \mathbb{I}_{\lbrace a^u(2) = i \rbrace} \dots + \mathbb{I}_{\lbrace a^u(n) = i \rbrace}, \\
= & \mathbb{I}_{\lbrace a^u(1) = i \rbrace} + \mathbb{I}_{\lbrace a^1(1) = i \rbrace} \dots + \mathbb{I}_{\lbrace a^1(n-1) = i \rbrace}, \\
\leq & 1 + T^1_i(n-1).
\end{align*}
Therefore, we obtain the following:
\begin{align}
\sum\limits_{v=1}^m T_i^v(n) &= T_i^1(n) + T_i^2(n) \dots + T_i^m(n), \notag \\
& \leq T_i^1(n) + 1 + T_i^1(n-1) \dots + 1 + T_i^1(n-1), \notag \\
& \leq (m-1) + m T_i^1(n), \label{eq: clever1}
\end{align}
since $T_i^1(n-1) \leq T_i^1(n)$. Now, we find an upper bound on $T_i^1(n)$ under FYL policy. Let $\tau_1$ be the least time step at which $m^1_i(\tau_1)$ is atleast $l_i = \frac{8 \ln n}{\Delta_i^2}$. Observe that, under FYL policy $T_i^1(\tau_1) = \lceil \frac{l_i}{m} \rceil$. Since, the center node has chosen arm $i$ for $\lceil \frac{l_i}{m} \rceil$ times, $(m-1)$ leaf nodes must have also selected arm $i$ for the same number of times. This leads to $m^1_i(\tau_1) = l_i$. Let $B_i^1(t)$ be the event that node-1 chooses arm $i$ in round $t$. Hence,
\begin{align*}
T_i^1(n) = T_i^1(\tau_1) + \sum\limits_{t = \tau_1 + 1}^n \mathbb{I}_{B_i^1(t)} = \Big\lceil \frac{l_i}{m} \Big\rceil + \sum\limits_{t = \tau_1 + 1}^n \mathbb{I}_{B_i^1(t)}.
\end{align*}
By using the analysis in Theorem~\ref{Thm:3.1}, we obtain 
\begin{equation*}
\mathbb{E} \Big[ \sum\limits_{t = \tau_1 + 1}^n \mathbb{I}_{B_i^1(t)} \Big] \leq \frac{2}{4 \beta -1} + \frac{2}{(4 \beta -1)^2 \ln (1/ \beta)}.
\end{equation*}
Hence,
\begin{equation*}
\mathbb{E} [T_i^1(n)] \leq \Big\lceil \frac{l_i}{m} \Big\rceil + \frac{2}{4 \beta -1} + \frac{2}{(4 \beta -1)^2 \ln (1/ \beta)}.
\end{equation*}
From (\ref{eq: clever1}),
\begin{align*}
\sum\limits_{v=1}^m \mathbb{E} [ T_i^v(n) ] \leq & \frac{8 \ln n}{\Delta_i^2} + 2m - 1 +\\
& \frac{2m}{4 \beta -1} \left( 1 + \frac{1}{(4 \beta -1) \ln (1/ \beta)} \right),
\end{align*}    
where we have substituted $l_i = \frac{8 \ln n}{\Delta_i^2}$. Therefore, the expected regret of the FYL policy on an $m$-node star network upto $n$ number of rounds is upper bounded as:
\begin{multline*}
\mathbb{E}[R^G(n)] \leq \sum\limits_{i:\mu_i < \mu^*}^K \frac{8 \ln n}{\Delta_i} + \Big[ 2m - 1 + \frac{2 m}{4 \beta -1} \cdot \\
 \left( 1 + \frac{1}{(4 \beta -1) \ln (1/ \beta)} \right) \Big] \sum\limits_{j=1}^K \Delta_j,
\end{multline*}
which completes the proof. 
\end{proof}
\section*{Appendix F}
\label{FYL Generic Upper Bound}
\emph{Proof of Theorem~\ref{Thm:5.2}}.\\
\begin{proof}
Since the leader node (a node in the given dominating set) in a particular component uses samples only from its neighbours in the same component, we can upper bound the expected regret of each component using Theorem~\ref{Thm:5.1}. We get the desired result by adding the expected regrets of all the components.
\end{proof}
\end{document}